\documentclass{article}

\usepackage[preprint]{nips_2018}

\usepackage[utf8]{inputenc} 
\usepackage[T1]{fontenc}    
\usepackage{hyperref}       
\usepackage{url}            
\usepackage{booktabs}       
\usepackage{amsfonts}       
\usepackage{nicefrac}       
\usepackage{microtype}      
\usepackage{algorithm}
\usepackage{algpseudocode}

\usepackage{amsmath, amssymb}
\usepackage{mathtools}
\usepackage{amsthm}
\newtheorem{theorem}{Theorem}
\newtheorem{lemma}[theorem]{Lemma}

\newtheorem{defn}{Definition}[section]
\DeclarePairedDelimiter\abs{\lvert}{\rvert}%
\DeclarePairedDelimiter\floor{\lfloor}{\rfloor}
\DeclarePairedDelimiter{\ceil}{\lceil}{\rceil}
\usepackage{subcaption}
\usepackage{graphicx}
\title{Online learning with feedback graphs \\ and switching costs}

\author{ Anshuka Rangi\\University of California, San Diego  \And Massimo Franceschetti\\
University of California, San Diego 
}

\begin{document}
\maketitle
\begin{abstract}
We study online learning when partial
feedback information is provided following every action of the learning process, and the learner incurs switching costs for changing his actions.  In this setting, the feedback information system  can be represented by a graph, and previous works studied the expected regret of the learner in the case of  a clique (Expert setup), or disconnected single loops (Multi-Armed Bandits (MAB)). This work provides a lower bound on the expected regret in the Partial Information (PI) setting, namely for general feedback graphs --excluding the clique. Additionally, it shows that all  algorithms  that are optimal without switching costs  are necessarily  sub-optimal in the presence of switching costs, which motivates the need to design new algorithms.  
We propose two new algorithms: Threshold Based EXP3 and EXP3.SC. 
For the two special cases of  symmetric PI setting and MAB, the expected regret of both of these algorithms is order optimal in the duration of the learning process. 
Additionally,  Threshold Based EXP3 is order optimal in the switching cost, whereas EXP3.SC   is not. 
Finally, empirical evaluations show that Threshold Based EXP3 outperforms the  previously proposed order-optimal algorithms   EXP3 SET in the presence of switching costs, and   Batch EXP3   in the MAB setting with switching costs.
\end{abstract}

\section{Introduction}
Online learning has a wide variety of applications like classification, estimation, and ranking, and it has been investigated in different areas, including learning theory, control theory, operations research, and statistics. The problem can be viewed as  a one player game against an adversary.  The game runs for $T$ rounds and at each round  the player chooses an action   from a given set of  $K$ actions. Every action $k\in [K]$
performed at   round  $t\in [T]$
carries a loss, that is a real number in the interval $[0,1]$. The losses for all pairs $(k,t)$  are assigned   by the adversary before the game starts.
The player also  incurs  a fixed and   known Switching Cost (SC)  every time he changes his action, that is an arbitrary real number $c>0$. 
The expected regret is   the 
expectation of  the sum of losses associated to the actions performed by the player plus the SCs minus  the losses incurred by the best fixed action in   hindsight.
The goal of the player is to minimize the expected regret over the duration of the game.

Based on the feedback information received after each action, online learning can be divided into three categories:
Multi-Armed Bandit (MAB),  Partial Information (PI), and Expert setting. In a MAB setting, at any given round the player only incurs the loss corresponding to the selected action, which implies the player only observes the loss of the selected action. In a PI setting, the player incurs the loss of the selected action $k \in [K]$, as well as  observes the losses that he would have incurred in that round by taking  actions in a subset of $[K]\backslash\{k\}$. This feedback system
can be viewed as a time-varying directed graph $G_t$ with $K$ nodes, where a directed edge $k \to j$ in $G_t$ indicates that performing an action $k$  at round $t$ also reveals the loss that the player would have incurred if action $j$ was  taken at round $t$. In an Expert setting,  taking an action reveals the losses  that the player would have incurred by taking any of  the   other actions in that round.  In this extremal case, the feedback system $G_t$ corresponds to a time-invariant, undirected clique. 

Online learning with PI has been used to design a variety of   systems \cite{gentile2014multilabel,katariya2016dcm,zong2016cascading,rangi2018distributed}. 
In these applications,  feedback   captures the idea of side information provided to the player during the   learning process. For example,   the performance of an employee can provide   information about the performance of other employees with similar skills, or the rating  of a  web  page  can provide  information on   ratings of  web pages with similar content.   
In most of these applications, switching between the actions is not free. For example, a company incurs a cost associated to the   learning phase while shifting an employee among different tasks, or switching the content of a web page frequently can  exasperate users  and force them to avoid visiting it. Similarly, re-configuring the production line in a factory  is  a costly process, and changing the  stock allocation in an investment portfolio is subject to certain fees. 
Despite the many  applications where both SC and PI are  an integral part of the learning process, the study of 
online learning with SC has been limited only to the MAB and Expert settings.
In the MAB setting, it has been shown that the expected regret of any player is at least $\tilde{\Omega}(c^{1/3} K^{1/3}T^{2/3})$ \cite{dekel2014bandits},  and that Batch EXP3  is an order optimal algorithm  \cite{arora2012online}. In the Expert setting, it has been shown that the expected regret is at least $\tilde{\Omega}(\sqrt{\log(K)T})$  \cite{cesa2006prediction}, and order optimal algorithms   have been proposed in \cite{geulen2010regret,gyorgy2014near}.  
The PI setup has been investigated only in the absence of SC, and for any fixed feedback system $G_{t}=G$ with independence number  $\alpha(G)>1$,  it has been shown that  the expected regret is at least $\tilde{\Omega}(\sqrt{\alpha(G)T})$~\cite{mannor2011bandits}. 

%
%


\begin{table*}[t]
\begin{center}
 \scriptsize{
\begin{tabular}{llll}
\textbf{Scenarios}  &\textbf{Threshold based EXP3} &\textbf{EXP3.SC} &\textbf{Lower Bound}\\
\hline \\
For all $t$, $G_{t}=G$           & $\tilde{O}(c^{1/3}(\mbox{mas}(G))^{1/3} T^{2/3})$ &$\tilde{O}(c^{4/3}(\mbox{mas}(G))^{2/3} T^{2/3})$ & $\tilde{\Omega}(c^{1/3}\alpha(G)^{1/3}T^{2/3})$\\
Symmetric PI &$\tilde{O}(c^{1/3}\alpha(G)^{1/3}T^{2/3})$  &$\tilde{O}(c^{4/3}\alpha(G)^{2/3}T^{2/3})$ & $\tilde{\Omega}(c^{1/3}\alpha(G)^{1/3}T^{2/3})$\\
MAB  &$\tilde{O}(c^{1/3}K^{1/3}T^{2/3})$  &$\tilde{O}(c^{4/3}K^{2/3}T^{2/3})$ & $\tilde{\Omega}(c^{1/3} K^{1/3}T^{2/3})$\\
$G_{1:T}$&$\tilde{O}(c \sum_{t=1}^{t^{*}}{\mbox{mas} (G_{(t)})/\mbox{mas}(G_{(T)})})$  &$\tilde{O}( \sum_{t=1}^{n^{*}}{\mbox{mas} (G_{(t)})/\mbox{mas}(G_{(T)})})$&$\tilde{\Omega}(c^{1/3}\beta(G_{1:T})^{1/3}T^{2/3})$ \\
Equi-informational &$\tilde{O}(c^{1/3}\alpha(G_{1})^{1/3}T^{2/3})$  &$\tilde{O}(c^{4/3}\alpha(G_{1})^{2/3}T^{2/3})$ &$\tilde{\Omega}(c^{1/3}\beta(G_{1:T})^{1/3}T^{2/3})$\\
\end{tabular}
}
\end{center}

\caption{Comparison of Threshold based EXP3 and EXP3.SC.} \label{contribution}

\end{table*}
 \subsection{Contributions}
 We provide a lower bound on the expected regret for any sequence of feedback graphs $G_{1},\ldots G_{T}$ in the PI  setting with SC. 
 We show that for any sequence of feedback graphs $G_{1:T}=\{G_{1},\ldots G_{T}\}$ with  independence sequence number $\beta(G_{1:T})>1$, the expected regret of any player is at least $\tilde{\Omega}(c^{1/3}\beta(G_{1:T})^{1/3}T^{2/3})$. We then   show that for 
 $G_{1:T}$ with $\alpha(G_{t})>1$ for all $t\leq T$, the expected regret of any player is at least $\tilde{\Omega}(c^{1/3}\sum_{G_{j}\in \mathcal{G}}\alpha(G_{j})^{1/3}N(G_{j})^{2/3})$, where $\mathcal{G}$ is the set of unique feedback graphs in the sequence $G_{1:T}$, and $N(G_{j})=\sum_{t=1}^{T}\mathbf{1}(G_{t}=G_{j})$ is the number of rounds for which the feedback graph $G_{j}$ is seen in $T$ rounds.  These results introduce a new figure of merit $\beta(G_{1:T})$ in the PI setting, which  can also be used to generalize the lower bound given in the PI setting without SC ~\cite{mannor2011bandits}. A consequence of these results is that   the presence of SC  changes the asymptotic regret by at least a factor $T^{1/6}$. Additionally, these results also recover the lower bound on the expected regret in the MAB setting  \cite{dekel2014bandits}. 
 
We also show that in the PI setting for any algorithm that is order optimal  without SC, there exists an assignment of losses from the adversary that forces the algorithm
to make at least $\tilde{\Omega}(T)$ switches, thus increasing its asymptotic  regret   by at least a factor $T^{1/2}$.  This shows  that   any algorithm that is order optimal in the PI setting without   SC,   is necessarily sub-optimal in the presence of SC, and  motivates the  development of  new  algorithms  in the PI setting and in the presence of SC.  

We propose two new algorithms for the PI setting with SC: 
 Threshold-Based EXP3 and EXP3.SC. Threshold-Based EXP3 requires the knowledge of $T$ in advance, whereas  EXP3.SC does not. The   performance of these algorithms is given for different scenarios in Table \ref{contribution}.  The  algorithms are order optimal in $T$ and $\beta(G_{1:T})$ for two special cases of feedback information system: symmetric PI setting i.e. the feedback graph $G_{t}=G$ is fixed and un-directed, and MAB. In these two cases,   $\beta(G_{1:T})$ equals $\alpha(G)$ and $K$ respectively. The  state-of-art algorithm EXP3 SET in PI setting without SC  is known to be order optimal only for these cases as well  \cite{alon2017nonstochastic}.   Threshold Based EXP3 is order optimal in the SC $c$ as well, while EXP3.SC  has an additional factor of $c$ in its expected regret. In the time-varying case, for sequence 
 $G_{1:T}$, the expected regret  is dependent on the worst $t^{*}$ and $n^{*}$ instances of the ratio of $\mbox{mas}(G_{t})$ 
 and  $\mbox{mas}(G_{(T)})$, where  $\{\mbox{mas}(G_{(1)}),   \mbox{mas}(G_{(2)}),  \ldots, \mbox{mas}(G_{(T)}) \}$ are the sizes of the maximal acyclic subgraphs of  $G_{1:T}$   arranged in non-increasing order, $t^{*}={\ceil{T^{2/3}c^{-2/3}\mbox{mas}^{1/3}(G_{(T)})}}$ and $n^*= 0.5\mbox{mas}^{1/3}(G_{(T)})T^{2/3}c^{1/3}$. Finally, Table \ref{contribution} also provides the performance in  the equi-informational setting, namely when  $G_{t}$ is undirected and  all the maximal acyclic subgraphs in $G_{1:T}$ have  the same size. The proofs of all these results are available online \cite{rangi2018online}. 

Numerical comparison 
shows that Threshold Based EXP3 outperforms   EXP3 SET   in the presence of SCs. Threshold Based EXP3 also outperforms Batch EXP3, which is another order optimal algorithm for the MAB setting with SC~\cite{arora2012online}.

\subsection{Related Work}
In the absence of SC, the lower bound on the expected regret  is known for all   three categories of online learning problems. In the MAB setting, the expected regret is at least $\tilde{\Omega}(\sqrt{KT})$~\cite{auer2002nonstochastic,cesa2006prediction,rangi2018unifying}. In the PI setting with fixed feedback graph $G$, the expected regret is at least $\tilde{\Omega}(\sqrt{\alpha(G)T})$~\cite{mannor2011bandits}. In the Expert setting, the expected regret is at least $\tilde{\Omega}(\sqrt{\log(K)T})$ \cite{cesa2006prediction}. All three cases present an asymptotic regret factor $T^{1/2}$. In contrast, in the presence of SC the expected regrets for MAB and Expert settings present different factors, namely $T^{2/3}$ and $T^{1/2}$ respectively. 
The expected regret is at least $\tilde{\Omega}(c^{1/3}K^{1/3}T^{2/3})$ in the MAB setting and $\tilde{\Omega}(\sqrt{\log(K)T})$  in the Expert setting \cite{dekel2014bandits}. This work provides the lower bound  on  the expected regret $\tilde{\Omega}(c^{1/3}\beta(G_{1:T})^{1/3}T^{2/3})$ for the  PI setting  in the presence of SC. For the case without SC, this work establishes that the lower bound in PI setting is $\tilde{\Omega}(\sqrt{\beta(G_{1:T})T})$.

The PI setting  was first considered in \cite{alon2013bandits,mannor2011bandits}, and many of its variations have been   studied without SC  
\cite{alon2015online, alon2013bandits,caron2012leveraging,rangi2018decentralized,langford2008epoch,kocak2016online,rangi2018consensus,wu2015online,rangi2018multi}.
In the adversarial setting we described, all of these algorithms are order optimal in the MAB and symmetric PI settings, but they  also require the player to have knowledge of the graph $G_t$ before performing an action. 
The algorithm EXP3 SET does not require such knowledge \cite{alon2017nonstochastic}. 
We show that all of these algorithms are sub-optimal in the PI setting with SC, and propose new algorithms that are order optimal in the MAB and  symmetric PI settings. 

In the expert setting with SC, there are two order optimal algorithms with expected regret $\tilde{O}(\sqrt{\log(K)T})$  \cite{geulen2010regret,gyorgy2014near}. In the MAB setting with SC, Batch EXP3 is an order optimal algorithm with expected regret $\tilde{O}(c^{1/3}K^{1/3}T^{2/3})$ \cite{arora2012online}. This algorithm has also been used  to solve a variant of the MAB setting  \cite{feldman2016online}. In the MAB setting, our algorithm has the same order of expected regret as Batch EXP3  but  it numerically outperforms Batch EXP3.

 There is a large  literature on a continuous variation of the MAB setting, where the number of actions $K$ depends on the number of rounds $T$. In this setting, the case without the SC was investigated in \cite{auer2007improved,bubeck2011x,kleinberg2005nearly,yu2011unimodal}. Recently,  the case  including  SC has also been studied  in \cite{koren2017bandits,koren2017multi}. In \cite{koren2017bandits},
 the algorithm Slowly Moving Bandits (SMB) has been proposed and  
in \cite{koren2017multi},  it has been extended to different settings.
These algorithms  incur an expected regret linear in $T$    when  applied in our discrete setting. 

\section{Problem Formulation}
Before the game starts, the adversary fixes a loss sequence $\ell_{1},\ldots,\ell_{T}\in [0,1]^{K}$,    assigning a loss in $[0,1]$ to $K$ actions for $T$ rounds. At round $t$, the player performs an action $i_{t}\in [K]$, and incurs the loss $\ell_{t}(i_{t})$ assigned by the adversary. If $i_{t}\neq i_{t-1}$, then the player also  incurs a  cost $c >0$ in addition to the loss $\ell_{t}(i_{t})$. 

In the PI setting, the feedback system can be viewed as a time-varying directed graph $G_t$ with $K$ nodes, where a directed edge $k \to j$ indicates that choosing action $k$  at round $t$ also reveals the loss that the player would have incurred if  action $j$ were  taken at round $t$. Let $S_{t}(i)=\{j: i\to j \mbox{ is a directed edge in }G_t\}$. Following the action $i_{t}$,  the player observes the losses he would have incurred in round $t$ by performing  actions in the subset $S_{t}(i_{t})\subseteq [K]$. Since the player always observes its own loss,  $i_{t}\in S_{t}(i_{t})$. 
In a MAB setup, the  feedback graph $G_{t}$   has only self loops, i.e. for all $t\leq T$ and $i\in[K]$, $S_{t}(i)=\{i\}$. In an Expert setup, $G_{t}$ is a undirected clique i.e. for all $t\leq T$ and $i\in[K]$, $S_{t}(i)=[K]$ . The expected regret of a player's strategy $\delta$ is defined as 
\begin{equation}\label{eq:RegretDef}
\begin{split}
        R^{\delta}(\ell_{1:T},c)&=\mathbf{E}\left[\sum_{t=1}^{T}\ell_{t}(i_{t}) + \sum_{t=2}^{T} c\cdot\textbf{1}(i_{t-1}\neq i_{t})\right]\\
        &\qquad-\min_{k\in [K]}\sum_{t=1}^{T}\ell_{t}(k).
\end{split}
\end{equation}
 In words, the expected regret is   the 
expectation of  the sum of losses associated to the actions performed by the player plus the SCs minus  the losses incurred by the best fixed action in the hindsight, and
the objective of the player is to minimize the expected regret.

\section{Lower Bound in PI setting with SC}
We start by defining the independence sequence number for a sequence of graphs  $G_{1:T}$.
\begin{defn} Given $G_{1:T}$,
let $P(G_{t})$ be the set of all the possible independent sets of the graph $G_{t}$. The independence sequence number $\beta(G_{1:T})$ is  the largest cardinality among all intersections of the independent sets $s_{1}\cap s_{2}\cap \ldots \cap s_{T} $, where $s_{t}\in P(G_{t})$. Namely, 
\begin{equation}
{\beta(G_{1:T})}=\max_{s_{1}\in P(G_{1}),\ldots s_{T}\in P(G_{T})}\abs{s_{1}\cap s_{2}\cap \ldots \cap s_{T}}.
\label{eq:iss}
\end{equation}
\end{defn}

\begin{defn}
The independence sequence set $\mathcal{I}(G_{1:T})$ is the set $s_1 \cap s_2 \cap \ldots s_T$ attaining the maximum in \eqref{eq:iss}.
\end{defn}

We use the notion of $\beta(G_{1:T})$ to provide a lower bound on the expected regret in the PI setting with SC.

\begin{theorem}\label{thm:lowerBound}
For any $G_{1:T}$   with  $\beta(G_{1:T})>1$, there exists a constant $b>0$ and an adversary's strategy (Algorithm \ref{alg:AdversaryStrategy}) such that for all $T\geq 27c\log_{2}^{3/2}(T)/\beta(G_{1:T})^{2}$, and  for any player's strategy $\mathcal{A}$, the expected regret of $\mathcal{A}$ is at least $b \, c^{1/3}\beta(G_{1:T})^{1/3}T^{2/3}/ \log T$.
\end{theorem}

 The proof of Theorem \ref{thm:lowerBound} relies on Yao's minimax principle \cite{yao1977probabilistic}. A randomized adversary strategy is constructed such that the expected regret of a player, whose action at any round is a deterministic function of his past observations, is at least $b \, c^{1/3}\beta(G_{1:T})^{1/3}T^{2/3}/ \log T$. This adversary strategy is described in Algorithm \ref{alg:AdversaryStrategy}, and is a generalization of the one proposed to establish similar bounds in the MAB setup \cite{dekel2014bandits}. The generalization is different than the one proposed for the PI setting without SC \cite{mannor2011bandits}.
 \begin{algorithm}[t]
\begin{algorithmic}
\State Input:$T>0$, $G_{1:T}$ with $\beta(G_{1:T})>1$;
\State Set $\epsilon_{1}=\epsilon_{2}=c^{1/3}\beta(G_{1:T})^{1/3}T^{-1/3}/9\log_{2}(T)$ and $\sigma = 1/9\log_{2}(T)$.
\State Choose an arm $X\in \mathcal{I}(G_{1:T})$ uniformly at random
\State Draw $T$  variables  such that  $\forall t\leq T$, $y_{t} \sim \mathcal{N}(0,\sigma^2)$.
\State For all $1\leq t\leq T$ and $i\in [K]$, assign
\[\ell_{t}(i)= W_{t}+0.5-\epsilon_{1}\textbf{1} (X=i)+\epsilon_{2}\textbf{1} (i\notin \mathcal{I}(G_{1:T})),\]
\[\ell_{t}(i) = clip(\ell_{t}(i)),\]
where $clip(a)=\min\{\max\{a,0\},1\}$, For all $ t\leq T$ $W_{t} = W_{\rho(t)}+y_{t},$, $W_{0}=0$, $\rho(t)= t-2^{\delta(t)}$ and $\delta(t)=\max\{i\geq 0: 2^{i} \mbox{ divides } t\}.$
\State Output: loss sequence $\ell_{1:T}$.
\caption{Adversary's strategy}
\label{alg:AdversaryStrategy}
\end{algorithmic}
\end{algorithm}
Since  $G_{1:T}$ is known to the adversary, it computes the independence sequence set $\mathcal{I}(G_{1:T})$, and the cardinality of this set is  $\beta{(G_{1:T})}$. For all $t\leq T$ and $i,j\in \mathcal{I}(G_{1:T})$, there exists no edge in the graph $G_{t}$ between the actions $i$ and $j$. Thus, the selection of any action in $\mathcal{I}(G_{1:T})$ provides no information about the losses of the other actions in $\mathcal{I}(G_{1:T})$. The adversary selects the optimal action uniformly at random from $\mathcal{I}(G_{1:T})$, and  assigns an expected loss of $1/2-\epsilon_{1}$. The remaining actions in $\mathcal{I}(G_{1:T})$ are assigned an expected loss of $1/2$ . On the other hand, since $i\in [K]\backslash\mathcal{I}(G_{1:T})$ provides information about the losses of actions in $\mathcal{I}(G_{1:T})$,   action $i$ is assigned an expected loss of $1/2+\epsilon_{2}$ to compensate for this additional information. In practice, even a small bias $\epsilon_{2}$ compensates for the extra information provided by an action in $[K]\backslash\mathcal{I}(G_{1:T})$. 

In  the PI setup
without SC, for a fixed feedback graph $G_{t}=G$, the expected regret is at least $\tilde{\Omega}(\sqrt{\alpha(G)T})$ \cite{alon2017nonstochastic}. The lower bound  is provided only for a fixed feedback system, and the lower bound for a general time-varying feedback system $G_{1:T}$ is left
as an open question \cite{alon2017nonstochastic}. This also motivates the investigation of different graph theoretic measures to study the PI setting \cite{alon2017nonstochastic}. Theorem \ref{thm:lowerBound} provides a lower bound for a general time-varying feedback system $G_{1:T}$ for the PI setting in presence of SC. 
The lower bound is dependent on the independence sequence number $\beta(G_{1:T})$ of $G_{1:T}$. Thus, the ideas introduced in Theorem \ref{thm:lowerBound} can be extended to close this gap in the literature of PI setting without SC. 

\begin{lemma}\label{lemme:PIwithoutSC}
In the PI setting without SC, for any  $G_{1:T}$   with $\beta(G_{1:T})>1$, there exists a constant $b>0$ and an adversary's strategy such that  for any player's strategy $\mathcal{A}$, the expected regret of $\mathcal{A}$ is at least $b \,\sqrt{ \beta(G_{1:T})T}$. 
\end{lemma}

Using Theorem \ref{thm:lowerBound} and Lemma \ref{lemme:PIwithoutSC}, it can be concluded that the presence of SC  changes the asymptotic regret by at least a factor $T^{1/6}$. 
In  the MAB setup, $\beta(G_{1:T})=K$,  and Theorem \ref{thm:lowerBound} recovers the bounds   provided in \cite{dekel2014bandits}.  

We now focus on the assumption in Theorem \ref{thm:lowerBound}, i.e. $\beta(G_{1:T})>1$. This is satisfied in many networks of practical interest. For example, networks modeled as $p$-random graphs where $p$ is the probability of having edge between two nodes. The expected independence number of these graphs is $2\log(Kp)/p$\cite{coja2015independent}. Since the probability of each node being in independent set is same, the expected value of $\beta(G_{1:T})$ is $K(2\log(Kp)/Kp)^T$, and $Kp$ is the expected node degree which is usually a constant as $p$ is inversely proportional to $K$. This is greater than one for large values of $K$, and small values of $T$. 

Algorithm \ref{alg:AdversaryStrategy}   depends on the independence sequence set $\mathcal{I}(G_{1:T})$ whose cardinality is non-increasing in $T$. In such cases, the adversary can split the sequence of feedback graphs $G_{1:T}$ into multiple sub-sequences i.e. say $M$ sub-sequences such that $U_{1}=\{G_{t}: t\in T_{1}\subseteq[T]\}$\ldots $U_{M}=\{G_{t}: t\in T_{M}\subseteq[T]\}$, $[T]=\cup_{m\in [M]} T_{m}$, and for all $m_{1},m_{2}\in [M]$, $T_{m_{1}}\cap T_{m_{2}}$ is an empty set. For each sub-sequence $U_{m}$, compute the independence sequence set and assign losses independently of other sub-sequences according to Algorithm \ref{alg:AdversaryStrategy}. This adversary's strategy, which we call Algorithm 1.1, gives the following bound on the expected regret.

\begin{theorem}\label{thm:advlowerbound2}
For any split of $G_{1:T}$ into disjoint sub-sequences $U_{1},\ldots U_{M}$ with $\beta(U_{m})>1$ and $N(U_{m})\geq  27c\log_{2}^{3/2}(N(U_{m}))/\beta(U_m)^{2}$ $\forall m\in [M]$, there exists a constant $b>0$ and an adversary's strategy (Algorithm 1.1) such that for any player's strategy $\mathcal{A}$, the expected regret of $\mathcal{A}$ is at least $b \, c^{1/3}\sum_{m\in [M]}\beta(U_{m})^{1/3}N(U_{m})^{2/3}/ \log T$, where $N(U_{m})=\sum_{t=1}^{T}\mathbf{1}(G_{t}\in U_{m})$ is the length of sub-sequence $U_{m}$.
\end{theorem}

With the insight provided by Theorem \ref{thm:advlowerbound2}, the regret can be made large with an appropriate split of $G_{1:T}$ into sub-sequences. This can be formulated as a sub-modular optimization problem where the objective is:  
\begin{equation}\label{eq:optimization}
    \max_{\{U_{1},\dots , U_{M}\}}c^{1/3}\sum_{m\in [M]}\beta(U_{m})^{1/3}N(U_{m})^{2/3}/ \log T \\
\end{equation}
\begin{equation}
\begin{split}
       \mbox{subject to }&\sum_{m\in [M]}N(U_{m})=T,\\
       & \forall m_1,m_2\in [M], U_{m_1}\cap U_{m_2}=\phi.
\end{split}
\end{equation}
This can be solved using greedy algorithms developed in the context of sub-modular maximization. 

Until now, we have been focusing on designing an adversary's strategy for maximizing the regret for a given sequence of feedback graphs $G_{1:T}$. Now, we briefly discuss the case when $G_{1:T}$ can also be chosen by the adversary. If the adversary is not constrained about the choice of feedback graphs, then the feedback graph that maximizes the expected regret would be a feedback graph with only self loops, as this reveals the least amount of information.  If the adversary is constrained by the choice of independence number, i.e. for all $t\leq T$, $\alpha(G_{t})\leq H$, then the optimal value of (\ref{eq:optimization}) is achieved for a sequence of fixed feedback graphs i.e.  for all $t\leq T$, $\alpha(G_{t})= H$, which implies $\beta(G_{1:T})=H$.

We now  discuss the trade-off between the loss incurred  and the number of switches performed by the player.
\begin{lemma}\label{lemma:lowerBoundSwitches}
If the expected regret computed ignoring the SC  of  any algorithm $\mathcal{A}$ is $\tilde{O}((\beta(G_{1:T})^{1/2}T)^{\beta})$, then there exists a loss sequence $\ell_{1:T}$ such that $\mathcal{A}$ makes at least $\tilde{\Omega}[(\beta(G_{1:T})^{1/2} T)^{2(1-\beta)}]$ switches.
\end{lemma}
Along the same lines of Lemma 4, it can also be shown that if the expected number of switches of $\mathcal{A}$ is $\tilde{O}[(\beta(G_{1:T})^{1/2} T)^{2(1-\beta)}]$, then the expected regret without SC is at least $ \tilde {\Omega} ((\beta(G_{1:T})^{1/2}T) ^{\beta})$. This provides the lower bound on the expected regret given the SC is constrained by a fixed budget.
Using Lemma \ref{lemma:lowerBoundSwitches}, if the expected regret without SC of  $\mathcal{A}$ is $\tilde{O}(\sqrt{\beta(G_{1:T})T})$, then there exists a loss sequence that forces $\mathcal{A}$ to make at least $\tilde{\Omega}(T)$ switches. This implies the regret of $\mathcal{A}$ with the SC is linear in $T$. Thus, any algorithm that is order optimal without   SC, is necessarily sub-optimal in the presence of SC, which  motivates the  design of  new  algorithms  in our setting.

\section{Algorithms in PI setting with SC}
 In this section, we introduce the two algorithms Threshold Based EXP3 and EXP3.SC for an uninformed setting where  $G_{t}$ is only revealed after the 
 action $i_{t}$ has been performed.  This is common in a variety of applications. For instance, a user's selection of some product allows   to infer that the user might be interested in similar products. However, no action on the recommended products may mean that user might not be interested in the product, does not need it or did not check the products. Thus, the feedback is revealed only after the action has been performed. 
 


In Threshold Based EXP3 (Algorithm \ref{alg:ThresholdEXP3}), each action $i\in [K]$ is assigned a weight $w_{i,t}$ at round $t$. When the loss of action  $i$ is observed at round $t$, i.e.  $i\in S_{t}(i_{t})$, $w_{i,t}$ is computed by penalizing $w_{i,t-1}$ exponentially by the empirical loss $\ell_{t}(i)\textbf{1}(i\in S_{t}(i_{t}))/q_{i,t}$. 
 At round $t$, $p_{t}=\{p_{1,t},\ldots,p_{K,t}\}$ is the sampling distribution where $p_{i,t}=w_{i,t}/\sum_{i
\in [K]}w_{i,t}$. At round $t$, action $i_{t}$ is selected with probability $p_{i,t}$ if 
the threshold  event $E^{t}=E_{1}^{t}\cup E^{t}_{2}\cup E^{t}_{3}$ is true, where
\[E_{1}^{t}=\{t=1\},\]
\[E_{2}^{t}=\{r>\gamma_{t},\mbox{ where } \gamma_{t}={T^{1/3}c^{2/3}}/{\mbox{mas}(G_{(T)})^{1/3}}\},\]
\begin{equation}\label{eq:threshold}
    \begin{split}
        E_{3}^{t}&=\hspace{-3pt}\{\forall i\in [K]\backslash \{i_{t}\}, \hat{\ell}_{t-1}(i)\hspace{-3pt}+\hspace{-3pt}{\ell}^{\prime}_{t-1}(i)\hspace{-3pt}> {\epsilon_{t}}/{\eta}+{1}/{q_{i_{t},t-1}},\\
        &\qquad\mbox{and there exists an } i\in [K]\backslash \{i_{t}\} \mbox{ such that }\\
        &\qquad\hat{\ell}_{t-1}(i)+{\ell}^{\prime}_{t-1}(i)-{\ell}^{\prime}_{t-1}(i_{t})\leq \epsilon_{t}/\eta+1/q_{i_{t},t-1}\},
    \end{split}
\end{equation}
 and $\epsilon_{t}\geq\log(tc^2/\mbox{mas}(G_{(T)}))/3$ . The event $E^{t}$  contains two threshold conditions, one on the variable $r$ and the other on the empirical losses. 
\begin{algorithm}[t]
\begin{algorithmic}
\State Initialization: $\eta\in (0,1]$; For all $i\in [K]$, $w_{i,1}=1$, $\hat{\ell}_{0}(i)=0$ and ${\ell}^{\prime}_{0}(i)=0$; $r=1$; 
\For{$t=1,\ldots, T$  }
\If{$E^{t}_{1}$ or $E^{t}_{2}$ or
$E^{t}_{3}$ (see (\ref{eq:threshold}))}
\If{$t\neq 1$}
\State  $\hat{\ell}_{t}(i)= \hat{\ell}_{t-1}(i)+{\ell}^{\prime}_{t-1}(i)$
\State $w_{i,t}=w_{i,t-1}\exp{(-\eta {\ell}^{\prime}_{t-1}(i))}$
\EndIf
\State Update $p_{i,t}={w_{i,t}}/{\sum_{j\in [K]}w_{j,t}}$.
\State Choose $i_{t} = i $ with probability  $p_{i,t}$.
\State Set  $r=1$ and for all $i\in [K]$, set ${\ell}^{\prime}_{t}(i)=0$
\Else 
\State For all $i\in [K]$, $p_{i,t}= p_{i,t-1}$, $\hat{\ell}_{t}(i)= \hat{\ell}_{t-1}(i)$ 
\State and $w_{i,t}=w_{i,t-1}$;  $i_{t} = i_{t-1} $;$r=r+1$
\EndIf
\State For all $i\in S_{t}(i_{t}) $, observe the pair $(\ell_{t}(i),i)$.
\State For all $i\in [K]$, ${\ell}^{\prime}_{t}(i)={\ell}^{\prime}_{t-1}(i)+\ell_{t}(i)\textbf{1}(i\in S_{t}(i_{t}))/q_{i,t}$, where $q_{i,t}=\sum_{j:j\to i}p_{j,t}$
 \EndFor
 \caption{Threshold based EXP3}
 \label{alg:ThresholdEXP3}
 \end{algorithmic}
\end{algorithm}
The threshold  event  $E^{t}$ is critical in balancing the trade-off between the number of switches and the loss incurred by the player.  $E_{1}^{t}$ corresponds to the first selection of action, and  incurs no SC. In  $E_{2}^{t}$, the variable $r$ tracks the number of rounds (or time instances) since the event $E^{t}$ occurred last time. If the choice of a new action has not been considered for past $\gamma_{t}$ rounds, then $E_{2}^{t}$ forces the player to choose an action according to the updated sampling distribution $p_{t}$ at round $t$. The threshold condition in  $E^{t}_{2}$ ensures that the regret incurred due to the selection of a sub-optimal action does not grow
continuously while trying to save on the SC between the actions. The event $E_{2}^{t}$ is independent of the observed losses, and will occur at most $O(T^{2/3})$ times. Unlike event $E_{2}^{t}$, the event $E_{3}^{t}$ is dependent on the   losses  $\hat{\ell}_{t}(i)$ and  ${\ell}^{\prime}_{t}(i)$, for all $i\in [K]$. Each loss
$\hat{\ell}_{t}(i)$ tracks the total empirical loss of action $i$ observed until round $\sigma(t)-1$,  i.e.
\[\hat{\ell}_{t}(i)=\sum_{k=1}^{\sigma(t)-1}\ell_{k}(i)\textbf{1}(i\in S_{k}(i_{k}))/q_{i,k},\]
where $\sigma(t)=\max\{k \leq t: E^{k} \mbox{ is true }\}$ is the latest round $k^{*}\leq t$ at which  $E^{k^{*}}$ is true. On the other hand, each loss ${\ell}^{\prime}_{t}(i)$ represents the total empirical loss of action $i$ observed between rounds $\sigma(t)$ and $t$, i.e.
\[{\ell}^{\prime}_{t}(i)=\sum_{k=\sigma(t)}^{t}\ell_{k}(i)\textbf{1}(i\in S_{k}(i_{k}))/q_{i,k}.\]
 This loss tracks the total  empirical loss   observed   after the selection of  an action at time instance $\sigma(t)$. The event $E_{3}^{t}$ balances exploration and exploitation while taking into account the SC. In  $E_{3}^{t}$, the first condition 
ensures that the player has sufficient amount of information about the losses of all other actions before  exploitation is considered. 
Given sufficient exploration has been performed, the second condition  
triggers the exploitation. The selection of a new action is considered when the empirical loss ${\ell}^{\prime}_{t}(i_{t})$ incurred by  the current action $i_{t}$, following its selection at $\sigma(t)$, becomes significant in comparison to the total empirical loss $\hat{\ell}_{t}(i)+{\ell}^{\prime}_{t}(i)$ incurred by the other actions $i\in[K]\backslash\{i_{t}\}$. Since the total empirical loss  of an action $i$  increases with $t$,   it is desirable that the threshold ${\epsilon_{t}}/{\eta}+{1}/{q_{i_{t},t-1}}$  increases with $t$ as well. Since    the increment in  
${\ell}^{\prime}_{t-1}(i_{t-1})$ is bounded above by $1/q_{i,t-1}$ at round $t$, for all
$i\in [K] \backslash\{i_{t}\}$,  $E_{3}^{t}$ implies that
\begin{equation}\label{eq:E3}
\hat{\ell}_{t-1}(i)+{\ell}^{\prime}_{t-1}(i)-{\ell}^{\prime}_{t-1}(i_{t-1})\geq {\epsilon_{t}}/{\eta}.
\end{equation}
Thus,  $E^{t}_{3}$ ensures that the player reconsiders the action selection if the loss incurred due to the current selection becomes significant in comparison to the total empirical loss of other actions. The event also ensures that the loss incurred due to the current selection is sufficiently smaller than the total empirical loss of other actions (see (\ref{eq:E3})). The event ensures that the sampling distribution $p_{t}$ has changed significantly from the previous sampling distribution $p_{\sigma(t-1)}$ before selecting the action again. 
Thus, $E^{t}_{3}$ balances exploration and exploitation  based on the observed losses.   

Batch EXP3, the order optimal algorithm in MAB with SC,  is EXP3 performed in batches of $O(T^{1/3})$. A similar strategy to design an algorithm for the PI setting with SC will fail because unlike MAB setting, the feedback graph $G_{t}$ can change at every round $t$, and this requires an update of empirical losses based on $G_{t}$ at every round. In our algorithm, the
computation of empirical loss is dependent on $G_{t}$ via $q_{i,t}$. 
Additionally, Batch EXP3 does not utilize the information about the observed losses, which is captured in $E_{3}^{t}$. The following theorem presents the performance guarantees of our algorithm.
\begin{theorem} \label{eq:TheoremRegret}
The following statements hold for Threshold Based EXP3:  \\
$(i)$The expected regret without accounting for  SC    is 
\begin{equation}
\begin{split}
     &\mathbf{E}\Bigg[\sum_{t=1}^{T}\ell_{t}(i_{t})-\min_{k\in [K]}\sum_{t=1}^{T}\ell_{t}(k)\Bigg]\\
     &\leq \frac{\log(K)}{\eta} + \frac{\eta}{2}\sum_{t=1}^{t^{*}}\frac{T^{2/3}c^{4/3}\mbox{\emph{mas}}(G_{(t)})}{(1-1/e)\mbox{\emph{mas}}^{2/3}(G_{(T)})},
\end{split}
\end{equation}
where $t^{*}={\ceil{T^{2/3}c^{-2/3}\mbox{\emph{mas}}^{1/3}(G_{(T)})}}$. \\
$(ii)$ The expected number of switches   is
\begin{equation}
    \mathbf{E}\bigg[\sum_{t=2}^{T}\textbf{1}(i_{t-1}\neq i_{t})\bigg]\leq 2T^{2/3}c^{-2/3}\mbox{\emph{mas}}^{1/3}(G_{(T)}).
\end{equation}
$(iii)$ Letting $\eta=\log(K)/T^{2/3}c^{1/3}\mbox{\emph{mas}}^{1/3}(G_{(T)})$, the expected regret  (\ref{eq:RegretDef})  is at most
\begin{align}
    3T^{2/3}c^{1/3}\hspace{-2pt}\mbox{\emph{mas}}^{1/3}(G_{(T)}) \nonumber \\  +\frac{ec\cdot \log(K)}{2(e-1)\mbox{\emph{mas}}(G_{(T)})}\sum_{t=1}^{t^{*}}{\mbox{\emph{mas}}(G_{(t)})}.
\end{align}
$(iv)$ In a symmetric PI setting i.e.  for all $t\leq T$  $G_{t}$ is un-directed and fixed, the expected regret   (\ref{eq:RegretDef})   is at most
\begin{equation}
    {4}T^{2/3}c^{1/3}\alpha^{1/3}(G_{1})\log(K).
    \end{equation}
\end{theorem}

\begin{algorithm}[t]
\begin{algorithmic}
\State Initialization: For all $i\in [K]$,  ${\hat{\ell}}_{1}(i)=0$; $t=1$, $\epsilon_{t}=0.5c^{1/3}\mbox{mas}^{1/3}(G_{(T)})/t^{1/3}$, $\eta_{t}=\log(K)/t^{2/3}c^{1/3}\mbox{mas}^{1/3}(G_{(T)})$
\For{$t=1,\ldots, T$  }
\State For all $i\in[K]$, update:
\State \qquad ${p}_{t}(i)=\frac{\exp(-\eta_{t}\hat{L}_{t-1}(i))}{\sum_{j\in[K]}\exp(-\eta_{t}\hat{L}_{t-1}(j) )}$
\State Choose $i_{t} = i_{t-1} $ with probability $1-\epsilon_{t}$,
\State else,   $i_{t} = i $ with probability $\epsilon_{t}{p}_{i,t}$.
\State For all $i\in S_{t}(i_{t}) $, observe the pair $(\ell_{t}(i),i)$.
\State For all $i\in[K]$, update $\hat{L}_{t}(i)=\sum_{n=1}^{t}\hat{\ell}_{n}(i)$,
\State where ${\hat{\ell}}_{t}(i)={\ell}_{t}(i)\textbf{1}(i\in S_{t}(i_{t}))/q_{i,t}$ and 
\State $q_{i,t}=\sum_{j:j\to i}p_{j,t}$.
 \EndFor
 \caption{EXP3.SC}
 \label{alg:HorizonFree}
 \end{algorithmic}
\end{algorithm}

In the PI setting, $\mbox{mas}(G_{t})$ captures the information provided by the feedback graph $G_{t}$. As $\mbox{mas}(G_{t})$ increases, the information provided by $G_t$ about  the losses of actions decreases. 
The regret of the algorithm depends on the $O(T^{2/3})$ instances of $\mbox{mas}(G_{(t)})$ (see Theorem \ref{eq:TheoremRegret} $(i)$). This is because the algorithm makes a selection of a new action  $O(T^{2/3})$ times in expectation (see Theorem \ref{eq:TheoremRegret} $(ii)$), and $G_{t}$ is not available in advance to influence the selection of the action. Also, the ratio $\mbox{mas}(G_{(t)})/\mbox{mas}(G_{(T)})$ is bounded above by $K$ and has no affect on order of $T$. The bounds of the algorithm on the expected regret are tight in two special cases. In the symmetric PI setting, the expected regret of Threshold Based EXP3 is $\tilde{O}(T^{2/3}c^{1/3}\alpha^{1/3}(G_{1}))$ (see Theorem \ref{eq:TheoremRegret} $(iii)$), hence, the algorithm is order optimal. In the MAB setting, the expected regret of Threshold Based EXP3 is $\tilde{O}(T^{2/3}c^{1/3}K^{1/3})$, hence, the algorithm is order optimal. The  state-of-art algorithm for the case without SCs  is known to be order optimal only for these cases as well, and the key challenges for closing this gap are highlighted in the literature\cite{alon2017nonstochastic}. 

EXP3.SC (Algorithm  \ref{alg:HorizonFree}) is another algorithm in PI setting with SC. The key differences between Threshold based EXP3 and EXP3.SC are highlighted here. Unlike Threshold based EXP3, EXP3.SC does not require the knowledge of the number of rounds $T$. Threshold based EXP3 favors the selection of action at regular intervals based on the event $E^{t}$. On contrary, EXP3.SC chooses a new action with probability $\epsilon_{t}$ which is decreasing in $t$. Thus, the algorithm favors exploration in the initial rounds, and favors exploitation as $t$ increases. In Threshold based EXP3, the scaling exponent $\eta$ is a constant dependent on $T$. On contrary, in EXP3.SC, the scaling exponent $\eta_{t}$  is time-varying, and is decreasing in $t$. The following theorem provides the performance guarantees of EXP3.SC.  
\begin{theorem}
The expected regret  (\ref{eq:RegretDef}) of EXP3.SC is at most
\[1.5c^{4/3}\mbox{\emph{mas}}^{1/3}(G_{(T)})T^{2/3}+\frac{2 \log(K)}{\mbox{\emph{mas}}^{2/3}(G_{(T)})}\sum_{j=1}^{n^*}\mbox{\emph{mas}}(G_{(j)}),\]
where $n^*= 0.5\mbox{\emph{mas}}^{1/3}(G_{(T)})T^{2/3}c^{1/3}$.
\end{theorem}
In symmetric PI and MAB settings, the expected regret of EXP3.SC is $\tilde{O}(c^{4/3}\alpha^{2/3}(G_{1})T^{2/3})$ and  $\tilde{O}(c^{4/3}K^{2/3}T^{2/3})$ respectively. Hence, the algorithm is order optimal in $T$ and $\beta(G_{1:T})$, and  has an additional factor of $c$ in the performance guarantees. In EXP3.SC, the dependency on $T$ is removed at the expense of an additional factor of $c$ in its performance.

In an alternative setting where the number of switches are constraint to be $O(T^{2(1-\beta)})$, it can be shown using Lemma \ref{lemma:lowerBoundSwitches} that the expected regret without SC is at least $ \tilde {\Omega} ((\beta(G_{1:T})^{1/2}T) ^{\beta})$. The two algorithms in this setting  are also simple variations of our two algorithms: Threshold based EXP3 and EXP3.SC. Threshold based EXP3 can be adapted by using threshold $\gamma_{t}=O(T^{2\beta-1})$, $\epsilon_{t}=O(\log(t)/ {2\beta-1})$ and $\eta=O(T^{-\beta})$. EXP3.SC can be adapted by using  $\epsilon_{t} = O(t^{-(2\beta-1)})$ and $\eta_{t}=O(t^{-\beta})$. These adapted algorithms would be order optimal in MAB and symmetric PI settings as well. 
\section{Performance Evaluation}
In this section, we numerically compare the performance of Threshold based EXP3 with EXP3 SET and Batch EXP3 in PI and MAB setups with SC respectively. 
 We do not compare the performance of our algorithm with the ones proposed in the Expert setting with SC because in MAB and PI setups, the player needs to balance the exploration-exploitation trade-off, while in the Expert setting the player is only concerned about the exploitation. Hence, there is a fundamental  discontinuity  in the design of algorithms as we move from the Expert to the PI setting. This gap is also evident from the discontinuity in the lower bounds in these settings, for the Expert setting the expected regret is at least $\tilde{\Omega}(\sqrt{\log(K)T})$, while  for the PI setting the expected regret is at least $\tilde{\Omega}(\beta(G_{1:T})^{1/3}T^{2/3})$, for $\beta(G_{1:T})>1$ which excludes the clique feedback graph.

We evaluate these algorithms by simulations because in real data sets, the adversary's strategy is not necessarily unfavorable for the players. Hence, the trends in the performance can vary widely across different data sets. For this reason, in the literature  only algorithms in stochastic setups rather than adversarial setups are typically evaluated on real data sets \cite{katariya2016dcm,zong2016cascading}. In our simulations, the adversary uses the Algorithm \ref{alg:AdversaryStrategy}, and $c=0.35$. 

Figure \ref{fig:PI} shows that the Threshold based EXP3 outperforms EXP3 SET in the presence of SC. Additionally, the expected regret and the number of switches of EXP3 SET grow linearly with $T$. These observations are in line with our theoretical results presented in Lemma \ref{lemma:lowerBoundSwitches}. The results presented here are for $G_{t}=G$, $\alpha(G)=5$ and $K=25$. Similar trends were observed for different value of $\alpha(G)$ and $K$.  

Figure \ref{fig:MAB} shows that Threshold based EXP3 outperforms Batch EXP3 in MAB setup with SC. The gap in the performance of these algorithm increases with $T$ (Figure \ref{fig:MAB}(a)). Additionally, the number of switches performed by threshold based EXP3 is larger than the number of switches performed by Batch EXP3 (Figure \ref{fig:MAB}(b) and (d)). The former algorithm utilizes the information about the observed losses via $E_{3}^{t}$ to balance the trade off between the regret and the number of switches. On contrary, Batch EXP3 does not utilize any information from the observed losses, and switches the action only after playing an action $\tilde{O}(T^{1/3})$ times. Note that MAB setup reveals the least information about the losses, and performance gap due to utilization of this information is significant (Figure \ref{fig:MAB}). This gap in performance grows as  $\beta(G_{1:T})$ decreases.  

In summary, Threshold Based EXP3 outperforms both EXP3 SET and Batch EXP3 in PI and MAB settings with SC respectively. Threshold Based EXP3 fills a gap in the literature by providing a solution for the PI setting with SC, and improves upon the existing literature in the MAB setup. 
\begin{figure}[t]
   \centering
\begin{subfigure}[b]{0.23\textwidth}
\includegraphics[width=\textwidth]{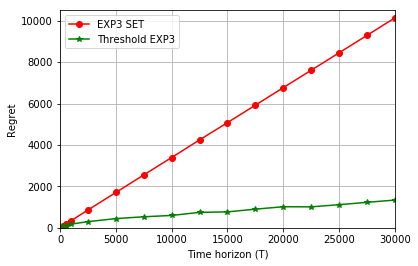}
\caption{For $\alpha(G)=5$}
\end{subfigure}
\begin{subfigure}[b]{0.23\textwidth}
\includegraphics[width=\textwidth]{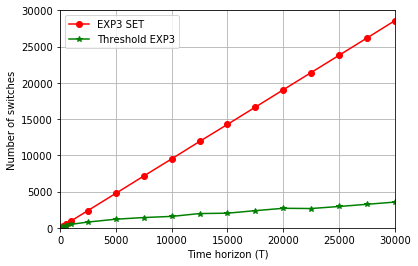}
\caption{For $\alpha(G)=5$}
\end{subfigure}
\caption{Performance evaluation of EXP3 SET and Threshold based EXP3 for K=25}
\label{fig:PI}
\end{figure}
\begin{figure}[t]
       \centering
\begin{subfigure}[b]{0.23\textwidth}
\includegraphics[width=\textwidth]{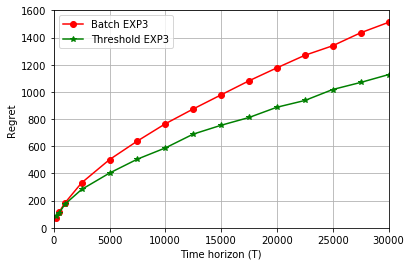}
\caption{For $K=5$}
\end{subfigure}
\begin{subfigure}[b]{0.23\textwidth}
\includegraphics[width=\textwidth]{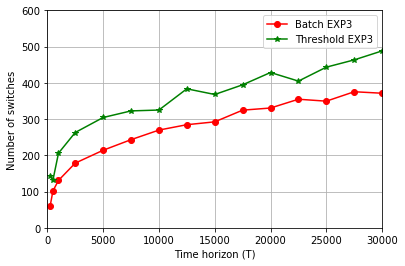}
\caption{For $K=5$}
\end{subfigure}
\begin{subfigure}[b]{0.23\textwidth}
\includegraphics[width=\textwidth]{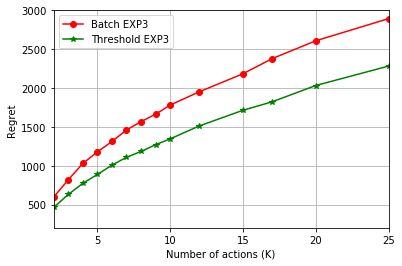}
\caption{For $T=20000$}
\end{subfigure}
\begin{subfigure}[b]{0.23\textwidth}
\includegraphics[width=\textwidth]{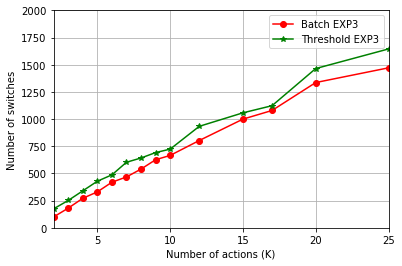}
\caption{For $T=20000$}
\end{subfigure}
\caption{Performance evaluation of Batch EXP3 and Threshold based EXP3 in MAB setting}
\label{fig:MAB}
\end{figure}

\section{Conclusion} 

This work focuses on online learning in the PI setting with SC in the presence of an adversary. The lower bound on the expected regret 
is presented in the  PI setup in terms of independence sequence number. 
There is a  need to design new algorithms in this setting because any algorithm that is order optimal without SC   is necessarily sub-optimal in the presence of SC. 
Two algorithms, Threshold Based EXP3 and EXP3.SC, are proposed and their  performance is evaluated in terms of expected regret. 
These algorithms are order optimal in $T$  in  two cases: symmetric PI and MAB setup.  
Numerical comparisons show that the Threshold Based EXP3 outperforms EXP3 SET and Batch EXP3 in PI setting with SC.

As future work, algorithms can be designed in a partially informed setting and a fully informed setting. In the partially informed setting, the feedback graph $G_{t}$ at round $t$ is revealed following the action at round $t-1$. Thus, the feedback graphs are revealed one at a time in advance at the beginning of each round. In the fully informed setting, the entire sequence of feedback graphs $G_{1:T}$ is revealed before the game starts. Since the adversary is aware of  $G_{1:T}$, these settings are important to study from the player's end as well. Note that without SC, the algorithms in both the partially informed and fully informed settings can exploit the feedback graphs at every round in a greedy manner, and perform an action accordingly. Hence, the algorithm in partially informed setting is also optimal in a fully informed setting in the absence of SC. On the contrary, in the presence of SC, a greedy exploitation of the feedback structure is not possible at every round. Hence, in fully informed setting with SC, the player chooses an action based on $G_{1:T}$ such that the selected action balances the trade off between the regret and the SC. Thus, the partially informed and fully informed settings of PI are of particular interest in the presence of SC, and is an interesting area for further study.

\medskip
\bibliographystyle{apalike}
\bibliography{BanditsWithSwitchingCost}
\small

\appendix
\section{Proof of Theorem 1}
\begin{proof}
 Without loss of generality, let the independent sequence set $\mathcal{I}(G_{1:T})$  formed of actions (or ``arms'') from $1$ to $\beta(G_{1:T})$. Given the sequence of feedback graphs $G_{1:T}$, let $T_{i}$ be the number of times the action $i\in \mathcal{I}(G_{1:T})=[\beta(G_{1:T})]$ is selected by the player in $T$ rounds. Let $T_{\Delta}$ be the total number of times the actions are selected from the set $[K]\backslash\mathcal{I}(G_{1:T})$. Let $\mathbb{E}_{i}$ denote expectation conditioned on $X=i$, and $\mathbb{P}_{i}$ denote the probability conditioned on $X=i$. Additionally, we define $\mathbb{P}_{0}$ as the probability  conditioned on event $\epsilon_{1}=0$. Therefore, under $\mathbb{P}_{0}$ , all the actions in the independent sequence set, i.e. $i\in \mathcal{I}(G_{1:T})$, incur an expected regret of $1/2$, whereas, the expected regret of actions $i\in [K]\backslash\mathcal{I}(G_{1:T})$ is $1/2+\epsilon_{2}$. Let $\mathbb{E}_{0}$ be the corresponding conditional expectation. For all $i\in[K]$ and $t\leq T$, $\ell_{t}({i})$ and $\ell^{c}_{t}({i})$ denote the unclipped and clipped loss of the action $i$ respectively.
 Assuming the unclipped losses are observed by the player, then $\mathcal{F}$ is the sigma field generated by the unclipped losses, and $S_{t}{(i_{t})}$ is the set of actions whose losses are observed at time $t$, following the selection of $i_{t}$, according to the feedback graph $G_{t}$. The observed sequence of unclipped losses will be referred as $\ell^{o}_{1:T}$. Additionally, $\mathcal{F}^{\prime}$ is the sigma field generated by the clipped losses, for all $t\in [T]$, $\ell^{\prime}_{t}(i)$ where $i\in S_{t}{(i_{t})}$, and the observed sequence of clipped losses will be referred as $\ell^{\prime o}_{1:T}$. By definition, $\mathcal{F}^{\prime}\subseteq \mathcal{F}$.
 
 Let $i_{1},\ldots,i_{T}$ be the sequence of actions selected by a player over the time horizon $T$. Then, the regret $R^{c}$ of the player corresponding to clipped loses is
\begin{equation}\label{eq:RegretClipped}
R^{c}=\sum_{t=1}^{T}\ell^{c}_{t}(i_{t})+c\cdot M_{s}-\min_{i\in [K]}\sum_{t=1}^{T}\ell^{c}_{t}(i),
\end{equation}
where $M_{s}$ is the number of switches in the  action selection sequence $i_{1},\ldots,i_{T}$, and $c$ is the cost of each switch in action. Now, we define the regret $R$ which corresponds to the unclipped loss function in Algorithm 1 as following
\begin{equation}\label{eq:RegretUnClip}
R=\sum_{t=1}^{T}\ell_{t}(i_{t})+c\cdot M_{s}-\min_{i\in [K]}\sum_{t=1}^{T}\ell_{t}(i).
\end{equation}
Using \cite[Lemma 4]{dekel2014bandits}, we have
\begin{equation}
\mathbb{P}\bigg(\mbox{For all }t\in[T], \frac{1}{2}+W_{t}\in \bigg[\frac{1}{6},\frac{5}{6}\bigg]\bigg)\geq \frac{5}{6}.
\end{equation}
Thus, for all $T>\max\{\beta(G_{1:T}),6\}$, we have $\epsilon_{1}=\epsilon_{2}<1/6$. If  $B=\{\mbox{For all } t\in[T]:{1}/{2}+W_{t}\in [{1}/{6},{5}/{6}]\}$ occurs and $\epsilon_{1}=\epsilon_{2}<1/6$, then for all $i\in [K]$,  $\ell^{c}_{t}(i)=\ell_{t}(i)$  which implies $R^{c}=R$ (see (\ref{eq:RegretClipped}) and (\ref{eq:RegretUnClip})). Now, if the event $B$ does not occur, then the losses at any time $t$ satisfy
\[\ell_{t}(i)-\ell^{c}_{t}(i)\leq (\epsilon_{1}+\epsilon_{2}).\]
Therefore, we have
\[c\cdot M_{s}\leq R^{c}\leq R \leq c\cdot M_{s}+(\epsilon_{1}+\epsilon_{2})T .\]
Now, for $T>\max\{\beta(G_{1:T}),6\}$, we have
\begin{equation}\label{eq:Regret}
    \mathbb{E}[R]-   \mathbb{E}[R^{c}]=(1-\mathbb{P}(B))\mathbb{E}[R-R^{c}|B \mbox{ does not occur}]\leq \frac{(\epsilon_{1}+\epsilon_{2})T}{6}.
\end{equation}
Thus, (\ref{eq:Regret}) lower bounds the actual regret $R^{c}$ in terms of regret $R$. Now, we derive the lower bound on regret $R$ corresponding to the unclipped loses. Using the definition of $R$, we have

\begin{equation}\label{eq:Regret1}
\begin{split}
&\mathbb{E}[R]=\max_{i\in [K]}\mathbb{E}[\sum_{t=1}^{T}\ell_{t}(i_{t})-\sum_{t=1}^{T}\ell_{t}(i)]+ \mathbb{E}[M_{s}]\\
&=\frac{1}{\beta(G_{1:T})}\sum_{i=1}^{\beta(G_{1:T})}\mathbb{E}_{i}[\sum_{t=1}^{T}\ell_{t}(i_{t})-\min_{i\in [K]}\sum_{t=1}^{T}\ell_{t}(i)]+ \mathbb{E}[M_{s}]\\
&=\frac{1}{\beta(G_{1:T})}\sum_{i=1}^{\beta(G_{1:T})}\mathbb{E}_{i}\bigg[\sum_{j\in\mathcal{I}(G_{1:T})\backslash\{i\}}\frac{1}{2}T_{j}+\bigg(\frac{1}{2}-\epsilon_{1}\bigg)T_i+\bigg(\frac{1}{2}+\epsilon_{2}\bigg)T_{\Delta}-\bigg(\frac{1}{2}-\epsilon_{1}\bigg)T\Bigg]+ \mathbb{E}[M_{s}]\\
&=\frac{1}{\beta(G_{1:T})}\sum_{i=1}^{\beta(G_{1:T})}\mathbb{E}_{i}\Bigg[\sum_{j=1}^{\beta(G_{1:T})}\frac{1}{2}T_{j}+\bigg(\frac{1}{2}+\epsilon_{2}\bigg)T_{\Delta}-\epsilon_{1}T_{i}-\bigg(\frac{1}{2}-\epsilon_{1}\bigg)T\Bigg]+ \mathbb{E}[M_{s}]\\
&\stackrel{(a)}{=}\frac{1}{\beta(G_{1:T})}\sum_{i=1}^{\beta(G_{1:T})}\mathbb{E}_{i}\Bigg[\epsilon_{2}T_{\Delta}+\epsilon_{1}(T-T_{i})\Bigg]+ \mathbb{E}[M_{s}]\\
&\stackrel{(b)}{\geq}\epsilon_{1}\Bigg(T-\frac{1}{\beta(G_{1:T})}\sum_{i=1}^{\beta(G_{1:T})}\mathbb{E}_{i}\big[T_{i}\big]+\mathbb{E}\big[T_{\Delta}\big]\Bigg)+\mathbb{E}[M_{s}],
\end{split}
\end{equation}
where $(a)$ follows from $\sum_{j=1}^{\beta(G_{1:T})}T_{j}+T_{\Delta}=T$, and $(b)$ follows from $\epsilon_1=\epsilon_2$.  

Now, we upper bound the $\mathbb{E}_{i}\big[T_{i}\big]$ in (\ref{eq:Regret1}) to obtain the lower bound on the expected regret $\mathbb{E}[R]$. Since the player is deterministic, the event $\{i_{t}=i\}$ is $\mathcal{F}^{\prime}$ measurable. Therefore, we have
\[\mathbb{P}_{i}(i_{t}=i)-\mathbb{P}_{0}(i_{t}=i)\leq d_{TV}^{\mathcal{F}^{\prime}}({P}_{0},{P}_{i})\stackrel{(a)}{\leq} d_{TV}^{\mathcal{F}}({P}_{0},{P}_{i}),\]

where $d_{TV}^{\mathcal{F}}({P}_{0},{P}_{i})=\sup_{A\in \mathcal{F}}\abs{\mathbb{P}_{0}(A)-\mathbb{P}_{i}(A)}$ is the total variational distance between the two probability measures, and $(a)$ follows from $\mathcal{F}^{\prime}\subseteq \mathcal{F}$.  Summing the above equation over $t\in[T]$ and $i\in \mathcal{I}(G_{1:T})$ yields
\[\sum_{i=1}^{\beta(G_{1:T})}\big(\mathbb{E}_{i}[T_{i}]-\mathbb{E}_{0}[T_{i}]\big)\leq T\cdot\sum_{i=1}^{\beta(G_{1:T})}d_{TV}^{\mathcal{F}}({P}_{0},{P}_{i}).\]
Rearranging the above equation and using $\sum_{i=1}^{\beta(G_{1:T})}\mathbb{E}_{0}[T_{i}]=\mathbb{E}_{0}[\sum_{i=1}^{\beta(G_{1:T})}T_{i}]=T$, we get
\[\sum_{i=1}^{\beta(G_{1:T})}\mathbb{E}_{i}[T_{i}]\leq T\cdot\sum_{i=1}^{\beta(G_{1:T})}d_{TV}^{\mathcal{F}}({P}_{0},{P}_{i})+T.\]
Combining the above equation with (\ref{eq:Regret1}), we get
\begin{equation}\label{eq:RegretKLDiv}
\begin{split}
        \mathbb{E}[R]&\geq \epsilon_{1}T-\frac{\epsilon_{1}T}{\beta(G_{1:T})}\cdot\sum_{i=1}^{\beta(G_{1:T})}d_{TV}^{\mathcal{F}}({P}_{0},{P}_{i})-\frac{\epsilon_{1}T}{\beta(G_{1:T})}+\frac{\epsilon_1}{\beta(G_{1:T})}\sum_{i=1}^{\beta(G_{1:T})}\mathbb{E}_{i}\big[T_{\Delta}\big]+\mathbb{E}[M_{s}]\\
        &\stackrel{(a)}{\geq}\frac{\epsilon_{1}T}{2}-\frac{\epsilon_{1}T}{\beta(G_{1:T})}\cdot\sum_{i=1}^{\beta(G_{1:T})}d_{TV}^{\mathcal{F}}({P}_{0},{P}_{i})+{\epsilon_1}\mathbb{E}\big[T_{\Delta}\big]+\mathbb{E}[M_{s}],
\end{split}
\end{equation}
where $(a)$ uses the fact that $\beta(G_{1:T})>1$. Next, we upper bound the second term in the right hand side of (\ref{eq:RegretKLDiv}). Using Pinsker's inequality, we have
\begin{equation}\label{eq:pinsker}
d_{TV}^{\mathcal{F}}({P}_{0},{P}_{i})\leq \sqrt{\frac{1}{2}D_{KL}(\mathbb{P}_{0}(\ell^{o}_{1:T})||\mathbb{P}_{i}(\ell^{o}_{1:T}))},
\end{equation}
where $\ell^{o}_{1:T}$ are the losses observed by the player over the time horizon $T$. Using the chain rule of relative entropy to decompose $D_{KL}(\mathbb{P}_{0}(\ell^{o}_{1:T})||\mathbb{P}_{0}(\ell^{o}_{1:T}))$, we get
\begin{equation}\label{eq:KLChainRule}
\begin{split}
      D_{KL}(\mathbb{P}_{0}(\ell^{o}_{1:T})||\mathbb{P}_{i}(\ell^{o}_{1:T}))&=\sum_{t=1}^{T} D_{KL}(\mathbb{P}_{0}(\ell_{t}^{o}|\ell^{o}_{1:t-1})||\mathbb{P}_{i}(\ell^{o}_{t}|\ell^{o}_{1:t-1}))\\
      &=\sum_{t=1}^{T} D_{KL}(\mathbb{P}_{0}(\ell^{o}_{t}|\ell^{o}_{\rho^{*}(t)})||\mathbb{P}_{i}(\ell^{o}_{t}|\ell^{o}_{\rho^{*}(t)})),
\end{split}
\end{equation}
where $\rho^{*}(t)$ is the set of time instances $0\leq k\leq t$ encountered when operation $\rho(.)$ in Algorithm 1 is applied recursively to $t$.
Now, we deal with each term $D_{KL}(\mathbb{P}_{0}(\ell^{o}_{t}|\ell^{o}_{\rho^{*}(t)})||\mathbb{P}_{i}(\ell^{o}_{t}|\ell^{o}_{\rho^{*}(t)}))$ in the summation individually. For $i \in \mathcal{I}(G_{1:T})$, we separate this computation into four cases: $i_{t}$ is such that loss of action $i$ is observed at both time instances $t$ and $\rho(t)$ i.e. $i\in S_{t}(i_{t})$ and $i\in S_{t}(i_{\rho(t)})$; $i_{t}$ is such that loss of action $i$ is observed at time instance $t$ but not at time instance $\rho(t)$  i.e. $i\in S_{t}(i_{t})$ and $i\notin S_{t}(i_{\rho(t)})$; $i_{t}$ is such that loss of action $i$ is not observed at  time instance $t$ but is observed at time instance $\rho(t)$  i.e. $i\notin S_{t}(i_{t})$ and $i\in S_{t}(i_{\rho(t)})$; $i_{t}$ is such that loss of action $i$ is not observed at both time instances $t$ and $\rho(t)$ i.e. $i\notin S_{t}(i_{t})$ and $i\notin S_{t}(i_{\rho(t)})$. 

\textit{Case 1}: Since the loss of action $i$ is observed from  the independent sequence set $\mathcal{I}(G_{1:T})$ at both the time instances, the loss distribution for the action $i$ is $\ell^{o}_{t}(i)|\ell^{o}_{\rho^{*}(t)}\sim \mathcal{N}(\ell_{\rho(t)}(i),\sigma^{2})$ for both $\mathbb{P}_{0}$ and $\mathbb{P}_{i}$. For all $j\in[K]\backslash\mathcal{I}(G_{1:T})$, the loss distribution is $\ell^{o}_{t}(j)|\ell^{o}_{\rho^{*}(t)}\sim \mathcal{N}(\ell_{\rho(t)}(i)+\epsilon_{1}+\epsilon_{2},\sigma^{2})$ under both $\mathbb{P}_{0}$ and $\mathbb{P}_{i}$. 

\textit{Case 2}: Since the loss of action $i$ is observed from the independent sequence set $\mathcal{I}(G_{1:T})$ at time instance $t$ but not at $\rho(t)$, therefore, there exists an action  $k^{\prime}\in\mathcal{I}(G_{1:T})\backslash\{i\}$ from the independent sequence set which was observed at time instance $\rho(t)$. Then, the loss distribution for the action $i$ is $\ell_{t}^{o}(i)|\ell^{o}_{\rho^{*}(t)}\sim \mathcal{N}(\ell_{\rho(t)}(k^{\prime}),\sigma^{2})$ under $\mathbb{P}_{0}$, and $\ell^{o}_{t}(i)|\ell^{o}_{\rho^{*}(t)}\sim \mathcal{N}(\ell_{\rho(t)}(k^{\prime})-\epsilon_{1},\sigma^{2})$ under $\mathbb{P}_{i}$. For all $j\in[K]\backslash\mathcal{I}(G_{1:T})$, the loss distribution is $\ell^{o}_{t}(j)|\ell^{o}_{\rho^{*}(t)}\sim \mathcal{N}(\ell_{\rho(t)}(k^{\prime})+\epsilon_{2},\sigma^{2})$ under both $\mathbb{P}_{0}$ and $\mathbb{P}_{i}$. 

\textit{Case 3}:Since the action $i$ is observed from  the independent sequence set $\mathcal{I}(G_{1:T})$ at time instance $\rho(t)$ but not at $t$, therefore, there exists an action  $k^{\prime}\in\mathcal{I}(G_{1:T})\backslash\{i\}$ from the independent sequence set which was observed at time instance $t$. Then, the loss distribution for the arm $k^{\prime}$ is $\ell^{o}_{t}(k^{\prime})|\ell^{o}_{\rho^{*}(t)}\sim \mathcal{N}(\ell_{\rho(t)}(i),\sigma^{2})$ under $\mathbb{P}_{0}$, and $\ell^{o}_{t}(k^{\prime})|\ell^{o}_{\rho^{*}(t)}\sim \mathcal{N}(\ell_{\rho(t)}(i)+\epsilon_{1},\sigma^{2})$ under $\mathbb{P}_{i}$. For all $j\in[K]\backslash\mathcal{I}(G_{1:T})$, the loss distribution is $\ell^{o}_{t}(j)|\ell^{o}_{\rho^{*}(t)}\sim \mathcal{N}(\ell_{\rho(t)}(i)+\epsilon_{1}+\epsilon_{2},\sigma^{2})$ under both $\mathbb{P}_{0}$ and $\mathbb{P}_{i}$. 

\textit{Case 4}: Let $k^{*}$ be the arm from the independent sequence set observed at time instance $\rho(t)$. Since the arm $i$ is not observed from the independent sequence set $\mathcal{I}(G_{1:T})$ at the time instances $t$ and $\rho(t)$, therefore the loss distribution for all arms $k^{\prime}\in\mathcal{I}(G_{1:T})\backslash\{i\}$ is $\ell^{o}_{t}(k^{\prime})|\ell^{o}_{\rho^{*}(t)}\sim \mathcal{N}(\ell_{\rho(t)}(k^{*}),\sigma^{2})$ for both $\mathbb{P}_{0}$ and $\mathbb{P}_{i}$. For all $j\in[K]\backslash\mathcal{I}(G_{1:T})$, the loss distribution is $\ell^{o}_{t}(j)|\ell^{o}_{\rho^{*}(t)}\sim \mathcal{N}(\ell_{\rho(t)}(k^{*})+\epsilon_{2},\sigma^{2})$ under both $\mathbb{P}_{0}$ and $\mathbb{P}_{i}$. 

Therefore,  we have
\begin{equation}\label{eq:KLDivergence}
\begin{split}
D_{KL}(\mathbb{P}_{0}(\ell^{o}_{t}|\ell^{o}_{\rho^{*}(t)})||\mathbb{P}_{i}(\ell^{o}_{t}|\ell^{o}_{\rho^{*}(t)}))&=\mathbb{P}_{0}(i\in S_{t}{(i_{t})},i\notin S_{\rho(t)}({i_{\rho(t)}}))\cdot D_{KL}(\mathcal{N}(0,\sigma^2)||\mathcal{N}(-\epsilon_{1},\sigma^2)) \\
&+\mathbb{P}_{0}(i\notin S_{t}({i_{t}}),i\in S_{\rho(t)}({i_{\rho(t)}}))\cdot D_{KL}(\mathcal{N}(0,\sigma^2)||\mathcal{N}(\epsilon_{1},\sigma^2))\\
&=\frac{\epsilon_{1}^2}{2\sigma^2}\mathbb{P}_{0}(B_{t}),
\end{split}
\end{equation}
where $B_{t}=\{i\in S_{t}({i_{t}}),i\notin S_{\rho(t)}({i_{\rho(t)}})\cup i\notin S_{t}{(i_{t})},i\in S_{\rho(t)}{(i_{\rho(t)})}\}$. The event $B_{t}$ implies that at least one of the following events are true:\\
$(i)$ The player has switched  between the feedback systems $S_{t}{(k_{1})}$ and $S_{\rho(t)}{(k_{2})}$ such that $i\in S_{t}{(k_{1})}$ but $i\notin S_{\rho(t)}{(k_{2})}$ or vice-versa. \\
$(ii)$ The player did not change the selection of action, however, the feedback system has changed between $\rho(t)$ and $t$ such that $i$ has become observable or vice versa. This can occur only if the fixed action belongs to $[K]\setminus\mathcal{I}(G_{1:T})$. \\
Let $N_{i}$ be the number of times a player switches from the feedback system which includes $i$  to the feedback system which does not include $i$ and vice-versa.  Then, using (\ref{eq:KLChainRule}) and (\ref{eq:KLDivergence}), we have
\begin{equation}\label{eq:KlFinal}
    D_{KL}(\mathbb{P}_{0}(\ell^{o}_{1:T})||\mathbb{P}_{i}(\ell^{o}_{1:T}))\leq \frac{\epsilon_{1}^2\omega(\rho)}{2\sigma^2}\mathbb{E}_{0}[N_{i}+T_{\Delta}],
\end{equation}
where $\omega(\rho)$ is the width of process $\rho(.)$ (see Definition 2 in \cite{dekel2014bandits}) and is bounded above by $2\log_2(T)$. Combining (\ref{eq:pinsker}) and (\ref{eq:KlFinal}), we have
\begin{equation}\label{eq:TVFinal}
\sup_{A\in \mathcal{F}}(\mathbb{P}_{0}(A)-\mathbb{P}_{i}(A)) \leq \frac{\epsilon_{1}}{\sigma}\sqrt{\log_2(T)\mathbb{E}_{0}[N_{i}+T_{\Delta}]}.
\end{equation}
If $M_{s}\geq \epsilon_{1}T$, then $\mathbb{E}[R^{\prime}]>\epsilon_{1}T$. Thus, the claimed lower bound follows. Now, let us assume $M_{s}\leq \epsilon_{1}T$. For all $i\in\mathcal{I}(G_{1:T})$, we have
\begin{equation}
\begin{split}
\mathbb{E}_{0}[M_{s}]-\mathbb{E}_{i}[M_{s}]&=\sum_{m=1}^{\floor{\epsilon_{1}T}}\mathbb{P}_{0}(M_{s}\geq m)-\mathbb{P}_{i}(M_{s}\geq m))\\
&\leq \epsilon_{1}T\cdot d_{TV}^{\mathcal{F}}(\mathbb{P}_{0},\mathbb{P}_{i}).
\end{split}
\end{equation}
Using the above equation, we have
\begin{equation}\label{eq:NumSwitches}
\begin{split}
\mathbb{E}_{0}[M_{s}]-\mathbb{E}[M_{s}]&=\frac{1}{\beta(G_{1:T})}\sum_{i=1}^{\beta(G_{1:T})}(\mathbb{E}_{0}[M_{s}]-\mathbb{E}_{i}[M_{s}])\\
&\leq \frac{\epsilon_{1}T}{\beta(G_{1:T})}\sum_{i=1}^{\beta(G_{1:T})}d_{TV}^{\mathcal{F}}(\mathbb{P}_{0},\mathbb{P}_{i}).
\end{split}
\end{equation}
Now, combining (\ref{eq:Regret}), (\ref{eq:RegretKLDiv}), (\ref{eq:TVFinal})and (\ref{eq:NumSwitches}), we obtain
\begin{equation}
\begin{split}
\mathbb{E}[R^{\prime}]&\geq \frac{\epsilon_{1}T}{6}-\frac{\epsilon_{1}T}{\beta(G_{1:T})}\sum_{i=1}^{\beta(G_{1:T})}\frac{\epsilon_{1}}{\sigma}\sqrt{\log_2(T)\mathbb{E}_{0}[N_{i}+T_{\Delta}]}+{\epsilon_1}\mathbb{E}\big[T_{\Delta}\big] +c\cdot\mathbb{E}_{0}[M_{s}]\\
&\stackrel{(a)}{\geq} \frac{\epsilon_{1}T}{6}-\frac{\epsilon_{1}^2T}{\sigma\sqrt{\beta(G_{1:T})}}\sqrt{2\log_2(T)\mathbb{E}_{0}[M_{s}+T_{\Delta}]}+{\epsilon_1}\mathbb{E}\big[T_{\Delta}\big] +c\cdot\mathbb{E}_{0}[M_{s}]\\
&\stackrel{(b)}{\geq} \frac{\epsilon_{1}T}{6} -\frac{\epsilon_{1}^{4}T^2\log_{2}(T)}{c\cdot\sigma^2\beta(G_{1:T})}+{\epsilon_1}\mathbb{E}\big[T_{\Delta}\big] +c\cdot \bigg(\frac{\epsilon_{1}^{4}T^2\log_{2}(T)}{2c^2\cdot\sigma^2\beta(G_{1:T})}-\mathbb{E}_0\big[T_{\Delta}\big]\bigg),\\
&\stackrel{}{\geq} \frac{c^{1/3}\beta(G_{1:T})^{1/3}T^{2/3}}{54\log_{2}(T)}-\frac{c^{1/3}\beta(G_{1:T})^{1/3}T^{2/3}}{162\log_{2}(T)}+(\epsilon_1-c)\mathbb{E}_0\big[T_{\Delta}\big]\\
&\stackrel{(c)}{\geq}\frac{c^{1/3}\beta(G_{1:T})^{1/3}T^{2/3}}{81\log_{2}(T)},
\end{split}
\end{equation}
where $(a)$ follows from the concavity of $\sqrt{x}$ and $\sum_{i}^{\beta(G_{1:T})}N_{i}\leq 2M_{s}$, $(b)$ follows from the fact that the right hand side is minimized for $\sqrt{\mathbb{E}_{0}[M_{s}+T_{\Delta}]}=\epsilon^2T\sqrt{\log_{2}(T)}/2c\sigma\sqrt{\beta(G_{1:T})}$, and $(c)$ follows from the assumption $T\geq 27c\log_{2}^{3/2}(T)/\beta(G_{1:T})^{2}$, which implies $\epsilon_1\geq c$. The claim of the theorem  now follows.
\end{proof}

\section{Proof of Lemma 2}
We have that $\beta(G_{1:T})$ actions are non adjacent in the entire sequence of feedback graphs $G_{1:T}$. Let $1,2,\dots \beta(G_{1:T})$ belong to the $\mathcal{I}(G_{1:T})$. Then, the adversary selects an action uniformly at random from the set $\mathcal{I}(G_{1:T})$ say $j$, and assigns the loss sequence to action $j$ using independent Bernoulli random variable with parameter $0.5-\epsilon$, where $\epsilon=\sqrt{\beta(G_{1:T})/T)}$. For all $i\in \mathcal{I}(G_{1:T})/\{j\}$,  losses are assigned using independent Bernoulli random variable with parameter $0.5$.  For all $i\notin \mathcal{I}(G_{1:T})$, the losses are assigned using independent Bernoulli random variable with parameter $1$. The proof of the lemma follows along the same lines as in Theorem 5 in (\cite{alon2017nonstochastic}). 
\section{Proof of Theorem 3}
Proof of this theorem uses the results from Theorem 1. Since the loss sequence is assigned independently to each sub-sequence $U_{m}$ where $m\in [M]$. Using Theorem 1, there exists a constant $b_{m}$ such that
\begin{equation}
\begin{split}
    &\mathbf{E}\Bigg[\sum_{t=1}^{T}(\ell_{t}(i_{t})\mathbf{1}(G_{t}\in U_{m})+cW_{m}\Bigg]-\min_{i\in U_{m}}\sum_{t=1}^{T}(\ell_{t}(i)\mathbf{1}(G_{t}\in U_{m})\\ 
    &\geq b_{m}c^{1/3}\beta(U_{m})^{1/3}N(U_{m})^{2/3}/\log(T),
\end{split}
\end{equation}
where $W_{m}$ is number of switches performed within the sequence $U_{m}$. 
Since
\[\sum_{m\in[M]}W_m\leq \sum_{t=1}^{T}\mathbf{1}(i_{t}\neq i_{t-1}),\]
there exist a constant $b$ such that the expected regret of any algorithm $\mathcal{A}$ is at least 
\[b \, c^{1/3}\sum_{m\in [M]}\beta(U_{m})^{1/3}N(U_{m})^{2/3}/ \log T.\]
\section{Proof of Lemma 4}
\begin{proof}
The proof follows from contradiction and is along the same lines as the proof of Theorem 4 in \cite{dekel2014bandits}. Let $\mathcal{A}$ performs at most $\tilde{O}((\beta(G_{1:T})^{1/2}T)^{\alpha})$ switches for any sequence of loss function over $T$ rounds with $\beta+\alpha/2<1$.  Then, there exists a real number $\gamma$ such that $\beta<\gamma<1-\alpha/2$. Then, assign $c=(\beta(G_{1:T})^{1/2}T)^{3\gamma-2}$. Thus, the expected regret, including the switching cost, of the algorithm is 
\[\tilde{O}((\beta(G_{1:T})^{1/2}T)^{\beta}+(\beta(G_{1:T})^{1/2}T)^{3\gamma-2}(\beta(G_{1:T})T)^{\alpha})=\tilde{o}(\beta(G_{1:T})^{1/2}T)^{\gamma},\]
over a sequence of losses assigned by the adversary because $\beta<\gamma$ and $\alpha<2-2\gamma$. However, according to Theorem 1, the expected regret is at least $\tilde{\Omega}(\beta(G_{1:T})^{1/3}(\beta(G_{1:T})^{1/2}T)^{(3\gamma-2)/3}T^{2/3})=\tilde{\Omega}((\beta(G_{1:T}) T)^{\gamma})$. Hence, by contradiction, the proof of the lemma follows.

\end{proof}
\section{Proof of Theorem 5}
\begin{proof}
Let $t_{1},t_{2}\ldots ,t_{\sigma(T)}$ be the sequence of time instances at which the event $E^{t}$ occurs during the duration $T$ of the game. We define $\{r_{j}=t_{j+1}-t_{j}\}_{1\leq j\leq T }$ as the sequence of inter-event times between the events $E^{t}$.  Let $\mbox{mas}(G_{(1)}),\ldots, \mbox{mas}(G_{(T)})$ denote the sequence in the decreasing order of size of maximal acyclic graphs, i.e. $\mbox{mas}(G_{(1)})$ (or $\mbox{mas}(G_{(T)})$) is the maximum (or minimum) size of maximal acyclic graph observed in sequence $G_{1:T}=\{G_1 , \ldots G_T\}$. Using the definition of $E^{t}$,  note that $r_{j}$ is a random variable bounded by $T^{1/3}c^{2/3}/\mbox{mas}(G_{(T)})^{1/3}$. For all $1\leq j\leq \sigma(T)$, the ratio of total weights of actions at round $t_{j}$ and $t_{j+1}$ is 

\begin{equation}\label{eq:EXP2ratio}
\begin{split}
\frac{W_{t_{j+1}}}{W_{t_{j}}}&=\sum_{i\in [K]}\frac{w_{i,t_{j+1}}}{W_{t_{j}}}\\
&=\sum_{i\in [K]}\frac{w_{i,t_{j}}\exp(-\eta \ell^{\prime}_{t_{j}+r_{j}-1}(i))}{W_{t_{j}}}\\
&=\sum_{i\in [K]}p_{i,t_{j}}\exp(-\eta \ell^{\prime}_{t_{j}+r_{j}-1}(i)) \\
&\stackrel{(a)}{\leq} \sum_{i\in [K]} p_{i,t_{j}}\bigg(1-\eta \ell^{\prime}_{t_{j}+r_{j}-1}(i)+\frac{1}{2}\eta^2 \ell^{\prime2}_{t_{j}+r_{j}-1}(i)\bigg)\\
&= 1-\eta\sum_{i\in [K]} p_{i,t_{j}}\cdot \ell^{\prime}_{t_{j}+r_{j}-1}(i)+\frac{\eta^2}{2} \sum_{i\in [K]} p_{i,t_{j}} \cdot\ell^{\prime2}_{t_{j}+r_{j}-1}(i),
\end{split}
\end{equation}

where $(a)$ follows from the fact that, for all $x\geq 0$, $e^{-x}\leq 1-x-x^{2}/2$. Now, taking logs on both sides of (\ref{eq:EXP2ratio}), summing over $t_{1},t_{2},\ldots t_{\sigma(T)}$, and using $\log(1+x)\leq x$ for all $x>-1$, we get
\begin{equation}\label{eq:EXP3ratio}
    \log\frac{W_{t_{\sigma(T)+1}}}{W_{1}} \leq -\eta\sum_{j=1}^{\sigma(T)}\sum_{i\in [K]} p_{i,t_{j}}\cdot \ell^{\prime}_{t_{j}+r_{j}-1}(i)+\frac{\eta^2}{2}\sum_{j=1}^{\sigma(T)} \sum_{i\in [K]} p_{i,t_{j}}\cdot \ell^{\prime2}_{t_{j}+r_{j}-1}(i).
\end{equation}
For all actions $k^{\prime}\in [K]$, we also have
\begin{equation}\label{eq:EXP3LowerBound}
\log\frac{W_{t_{\sigma(T)+1}}}{W_{1}}\geq \log\frac{w_{k^{\prime},t_{\sigma(T)+1}}}{W_{1}}\geq -\eta\sum_{j=1}^{\sigma(T)} \ell^{\prime}_{t_{j}+r_{j}-1}(k^{\prime})-\log (K).
\end{equation}
Combining (\ref{eq:EXP3ratio}) and (\ref{eq:EXP3LowerBound}), for all $k^{\prime}\in [K]$, we obtain
\begin{equation}\label{eq:EXP3Final}
\sum_{j=1}^{\sigma(T)}\sum_{i\in [K]} p_{i,t_{j}}\cdot \ell^{\prime}_{t_{j}+r_{j}-1}(i)-\sum_{j=1}^{\sigma(T)} \ell^{\prime}_{t_{j}+r_{j}-1}(k^{\prime}) \leq \frac{\log(K)}{\eta} + \frac{\eta}{2}\sum_{j=1}^{\sigma(T)} \sum_{i\in [K]} p_{i,t_{j}}\cdot \ell^{\prime2}_{t_{j}+r_{j}-1}(i).
\end{equation}
Now, for all $i\in [K]$, the conditional expectation of $\ell^{\prime}_{t_{j}+r_{j}-1}(i)$ is
\begin{equation}\label{eq:Loss}
\begin{split}
    \mathbb{E}\bigg[\ell^{\prime}_{t_{j}+r_{j}-1}(i)\Big|p_{t_{j}},r_{j}\bigg]& = \sum_{t=t_{j}}^{t_{j}+r_{j}-1}\sum_{k^{\prime}:i\in S_{t}(k^{\prime})}p_{k^{\prime},t_{j}}\cdot \frac{\ell_{t}(i)}{q_{i,t}}, \\
    &=\sum_{t=t_{j}}^{t_{j}+r_{j}-1}\frac{\ell_t(i)}{q_{i,t}}\cdot\sum_{k^{\prime}:i\in S_{t}(k^{\prime})}p_{k^{\prime},t_{j}},\\
    &=\sum_{t=t_{j}}^{t_{j}+r_{j}-1}\ell_t(i).
\end{split}
\end{equation}
Therefore,  we have that  for all $i\in [K]$, the conditional expectation
\begin{equation}\label{eq:MinEq}
\mathbb{E}\bigg[\sum_{j=1}^{\sigma(T)} \ell^{\prime}_{t_{j}+r_{j}-1}(i)\Big|\{p_{t_{j}},r_{j}\}_{1\leq j\leq \sigma(T)}] \bigg]=\sum_{j=1}^{\sigma(T)}\sum_{t=t_{j}}^{t_{j}+r_{j}-1}\ell_t(i)=\sum_{t=1}^{T}\ell_t(i).
\end{equation}
Now, the expectation of second term in right hand side of (\ref{eq:EXP3Final}) is
\begin{equation}\label{eq:squareLoss}
\begin{split}
\mathbb{E} \left[\sum_{j=1}^{\sigma(T)}\sum_{i\in[K]}p_{i,t_{j}}\cdot\ell^{\prime2}_{t_{j}+r_{j}-1}(i)\right] &= \mathbb{E}\Bigg[\sum_{j=1}^{\sigma(T)}\mathbb{E}\bigg[\sum_{i\in[K]}p_{i,t_{j}}\ell^{\prime2}_{t_{j}+r_{j}-1}(i)|\{p_{t_{j}},r_{j}\}_{1\leq j\leq \sigma(T)}\bigg]\Bigg]\\
&\stackrel{(a)}{\leq}\mathbb{E}\Bigg[\sum_{j=1}^{\sigma(T)}\mbox{mas}(G_{t_j:t_{j}+r_{j}-1}){r^2_{j}}\Bigg],\\
\end{split}
\end{equation}
where $\mbox{mas}(G_{t_j:t_{j}+r_{j}-1})=\max_{n\in [t_{j},t_{j}+r_{j}-1]}\mbox{mas}(G_{n})$, and $(a)$ follows from the fact that, for all $i\in[K]$ and $t\leq T$, $\ell_{t}(i)\leq 1$, and  $\sum_{i\in[K]}p_{i,t}/q_{i,t}\leq \mbox{mas}(G_{t})$\cite[Lemma 10]{alon2017nonstochastic}. 

Now, we bound $\sum_{j=1}^{\sigma(T)}\mbox{mas}(G_{t_j:t_{j}+r_{j}-1}){r^2_{j}}$. We write the following optimization problem:
\begin{equation}
    \max_{\{r_{j}\}_{1\leq j\leq T}}\sum_{j=1}^{T}\mbox{mas}(G_{t_j:t_{j}+r_{j}-1}){r^2_{j}}, \mbox{ subject to}
\end{equation}
\[\sum_{j=1}^{T}r_{j}=T,\]
\[0\leq r_{j}\leq \frac{T^{1/3}c^{2/3}}{\mbox{mas}^{1/3}(G_{(T)})}.\]
Since the objective function is submodular and the constraints are linear, the ratio of the solution of the greedy algorithm and the optimal solution is  at most $(1-1/e)$ (\cite{nemhauser1978best}). Therefore, the optimal solution $o^{*}$ of the above optimization problem is 
\begin{equation}\label{eq:OptSolution}
    o^{*}\leq \sum_{t=1}^{t^{*}}\frac{T^{2/3}\mbox{mas}(G_{(t)})c^{4/3}}{(1-1/e)\mbox{mas}^{2/3}(G_{(T)})},
\end{equation}
where $t^{*}={\ceil{T^{2/3}c^{-2/3}\mbox{mas}^{1/3}(G_{(T)})}}$. Using (\ref{eq:EXP3Final}), (\ref{eq:Loss}), (\ref{eq:MinEq}), (\ref{eq:squareLoss}) and (\ref{eq:OptSolution}), we have 
\begin{equation}\label{eq:AggLoss}
    \mathbb{E}\Bigg[\sum_{j=1}^{\sigma(T)}\sum_{i\in [K]} p_{i,k_{j}} \sum_{t=k_{j}}^{k_{j}+r_{j}-1}\ell_t(i)-\sum_{j=1}^{T} \ell_{t}({k^{\prime}})\Bigg]\leq \frac{\log(K)}{\eta} + \frac{\eta}{2}\sum_{t=1}^{t^{*}}\frac{T^{2/3}c^{4/3}\mbox{mas}(G_{(t)})}{(1-1/e)\mbox{mas}^{2/3}(G_{(T)})}.
\end{equation}
Additionally, the player switches its action only if $E^{t}$ is true. Thus, using (\ref{eq:AggLoss}) and $c(i,j)=c$, for all $i,j\in[K]$, we have
\begin{equation}\label{eq:regret}
    R^{\mathcal{A}}(l_{1:T},\mathcal{C})\leq \frac{\log(K)}{\eta} + \frac{\eta}{2}\sum_{t=1}^{t^{*}}\frac{T^{2/3}c^{4/3}\mbox{mas}(G_{(t)})}{(1-1/e)\mbox{mas}^{2/3}(G_{(T)})}+ c\cdot\mathbb{E}[\sum_{t=2}^{T}\textbf{1}(i_{t}\neq i_{t-1})].
\end{equation}

Now, we bound $\mathbb{E}[\sum_{t=2}^{T}\textbf{1}(i_{t}\neq i_{t-1})]$. $E^{t}_{1}$ occurs with probability 1, and does not contribute to any SC. $E^{t}_{2}$ can lead to at most $\ceil{T^{2/3}c^{-2/3}\mbox{mas}^{1/3}(G_{(T)})}$ switches. Now, let $E^{t}_{3}$ causes $N_{T}$ switches. Then, we have
\begin{equation}
\begin{split}
\mathbb{E}[N_{T}]&=\mathbb{E} \left[\sum_{j=1}^{\sigma(T)}\textbf{1}(i_{t_{j+1}}\neq i_{t_{j}},E^{t_j}_{3} \mbox{ is true})\right]\\
&=\mathbb{E}\Bigg[\sum_{j=1}^{\sigma(T)}\mathbb{E}\bigg[\textbf{1}(i_{t_{j+1}}\neq i_{t_{j}}, E^{t_j}_{3} \mbox{ is true})\bigg|\{p_{t_{j}},r_{j}\}_{1\leq j\leq\sigma(T)}\bigg]\Bigg]\\
&\leq \mathbb{E}\Bigg[\sum_{j=1}^{\sigma(T)}\mathbb{E}\bigg[\sum_{i\in[K], k^{\prime}\in[K]\backslash \{i\} }\mathbb{P}(i_{t_{j}}=i\big|E^{t_j}_{3} \mbox{ is true})\mathbb{P}(i_{t_{j+1}}=k^{\prime}\big|i_{t_{j}}=i)\bigg|\{p_{t_{j}},r_{j}\}_{1\leq j\leq \sigma(T)}\bigg]\Bigg]\\
&= \mathbb{E}\bigg[\sum_{j=1}^{\sigma(T)}\sum_{i\in[K], k^{\prime}\in[K]\backslash \{i\} } p_{i,t_{j}}p_{k^{\prime},t_{j+1}}\bigg]\\
&\stackrel{(a)}{\leq} \sum_{t=1}^{T} c^{-2/3}\mbox{mas}^{1/3}(G_{(T)})t^{-1/3} =  c^{-2/3}\mbox{mas}^{1/3}(G_{(T)})T^{2/3},
\end{split}
\end{equation}
where $(a)$ follows from Lemma \ref{lemma:Event3} in this section. Thus, the number of switches are $2c^{-2/3}\mbox{mas}^{1/3}(G_{(T)})T^{2/3}$, and the SC is $2c^{1/3}\mbox{mas}^{1/3}(G_{(T)})T^{2/3}$.

Part $(iii)$ of the theorem follows by combining the results from $(i)$ and $(ii)$. Part $(iv)$ follows from the fact that if $G_{t}$ is undirected, $\mbox{mas}(G_{t})=\alpha(G_{t})$.
\end{proof}

\begin{lemma}\label{lemma:Event3} Given $i\in[K]$ is chosen at time instance $t_{j}$, for all $k^{\prime}\in [K]\backslash\{i\}$, we have
\[p_{i,t_{j}}\cdot p_{k^{\prime},t_{j+1}}\leq (t_{j+1})^{-1/3}.\]
\end{lemma}
\begin{proof}
Given $i$ is chosen at time instance $t_{j}$, for all $k^{\prime}\in[K]\backslash\{i\}$, we  have
\begin{equation}\label{eq:ProbBound}
\begin{split}
    \frac{p_{k^{\prime},t_{j+1}}}{p_{i,t_{j+1}}}&=\frac{p_{k^{\prime},1}\exp(-\eta \hat{\ell}_{t_{j+1}}(k^{\prime}))}{p_{i,t_{j}}\exp(-\eta {\ell}^{\prime}_{t_{j}+r_{j}-1}(i))}\\
    &\stackrel{(a)}{=}\frac{p_{k^{\prime},1}\exp(-\eta( \hat{\ell}_{t_{j}}(k^{\prime})+{\ell}^{\prime}_{t_{j}+r_{j}-1}(k^{\prime})))}{p_{i,t_{j}}\exp(-\eta {\ell}^{\prime}_{t_{j}+r_{j}-1}(i))}\\
     &\stackrel{(b)}{\leq} \frac{\exp\big(-\eta (\hat{\ell}_{t_{j}}(k^{\prime})+{\ell}^{\prime}_{t_{j}+r_{j}-1}(k^{\prime})-{\ell}^{\prime}_{t_{j}+r_{j}-1}(i))\big)}{ p_{i,t_{j}}} \\
     &\stackrel{(c)}{\leq} \frac{\exp\big(-\eta (\epsilon_{t_{j+1}}/\eta)\big)}{K p_{i,t_{j}}}\\
     &= \frac{\exp\big(- \epsilon_{t_{j+1}}\big)}{ p_{i,t_{j}}},
\end{split}
\end{equation}
where $(a)$ follows from the fact that $\hat{\ell}_{t_{j+1}}(k^{\prime})=\hat{\ell}_{t_{j}}(k^{\prime})+{\ell}^{\prime}_{t_{j}+r_{j}-1}(k^{\prime})$; $(b)$ follows from $p_{k^{\prime},1}=1/K$; $(c)$ follows from the fact that 
for all $k\in[K]\backslash\{i\}$, $\hat{\ell}_{k,t-1}-{\ell}^{\prime}_{i,t-1}>\epsilon_{t}/\eta$ as the increment in ${\ell}^{\prime}_{i,t-1}$ is bounded by $1/q_{i,t-1}$.
Now, replacing $\epsilon_{t}\geq\log(tc^2/\mbox{mas}(G_{(T)}))/3$ in  (\ref{eq:ProbBound}), we have 
\begin{equation} \label{eq:ProbOnSwitches}
p_{i,t_{j}}\cdot p_{k^{\prime},t_{j+1}}\leq c^{-2/3}\mbox{mas}^{1/3}(G_{(T)})t_{j+1}^{-1/3}.
\end{equation}
\end{proof}
\section{Proof of Theorem 6}
\begin{proof}
We borrow the notations from the proof of Theorem 5. Using the fact that $\eta_{t}$ is decreasing in $t$ and (\ref{eq:EXP3Final}), we have
\begin{equation}
\sum_{j=1}^{\sigma(T)}\sum_{i\in [K]} p_{i,t_{j}}\cdot \ell^{\prime}_{t_{j}+r_{j}-1}(i)-\min_{k^\prime \in [K]}\sum_{j=1}^{\sigma(T)} \ell^{\prime}_{t_{j}+r_{j}-1}(k^{\prime}) \leq \frac{\log(K)}{\eta_{T}} + \sum_{j=1}^{\sigma(T)} \frac{\eta_{t_j}}{2}\sum_{i\in [K]} p_{i,t_{j}}\cdot \ell^{\prime2}_{t_{j}+r_{j}-1}(i).
\end{equation}
Now, taking expectation on both the sides and using the fact that expectation of the $\min (.)$ is smaller than the $\min(.)$ of the expectation, we have
\begin{equation}\label{eq:horizonFree}
\begin{split}
    &\mathbf{E}\Bigg[\sum_{j=1}^{\sigma(T)}\sum_{i\in [K]} p_{i,t_{j}}\cdot \ell^{\prime}_{t_{j}+r_{j}-1}(i)\Bigg]-\min_{k^\prime \in [K]}\mathbf{E}\Bigg[\sum_{j=1}^{\sigma(T)} \ell^{\prime}_{t_{j}+r_{j}-1}(k^{\prime})\Bigg] \\
    &\leq \frac{\log(K)}{\eta_{T}} + \mathbf{E}\Bigg[\sum_{j=1}^{\sigma(T)} \frac{\eta_{t_j}}{2}\epsilon_{t_j}\mathbf{E}\Bigg[\sum_{i\in [K]} p_{i,t_{j}}\cdot \ell^{\prime2}_{t_{j}+r_{j}-1}(i)|p_{t_{j}},r_{j},\mathbf{1}(i_{t} \mbox{ is selected using } p_{t})\Bigg]\Bigg],\\
    &\stackrel{(a)}{\leq} \frac{\log(K)}{\eta_{T}} + \mathbf{E}\Bigg[\sum_{j=1}^{\sigma(T)} \frac{\eta_{t_j}}{2}\epsilon_{t_j}\mathbf{E}[\mbox{mas}(G_{t_j:t_{j}+r_{j}-1}){r^2_{j}}|\mathbf{1}(i_{t} \mbox{ is selected using } p_{t})]\Bigg],\\
    &\stackrel{(b)}{\leq}\frac{\log(K)}{\eta_{T}} + \mathbf{E}\Bigg[\sum_{j=1}^{\sigma(T)} \frac{\eta_{t_j}}{2}\epsilon_{t_j}\frac{2\cdot\mbox{mas}(G_{t_j:t_{j}+r_{j}-1})}{\epsilon_{t_j}^2}\Bigg],\\
     &\stackrel{}{=}\frac{\log(K)}{\eta_{T}} + \mathbf{E}\Bigg[\sum_{j=1}^{\sigma(T)} \frac{\eta_{t_j}}{2}\frac{2\cdot\mbox{mas}(G_{t_j:t_{j}+r_{j}-1})}{\epsilon_{t_j}}\Bigg],\\
     &\stackrel{(c)}{\leq}\frac{\log(K)}{\eta_{T}} + \mathbf{E}\Bigg[\sum_{j=1}^{\sigma(T)} \frac{2 \log(K)}{\mbox{mas}^{2/3}(G_{(T)})}{\mbox{mas}(G_{(j)})}\Bigg],\\
     &\stackrel{(d)}{\leq} \frac{\log(K)}{\eta_{T}} + \sum_{j=1}^{\mathbf{E}[\sigma(T)]} \frac{2 \log(K)}{\mbox{mas}^{2/3}(G_{(T)})}{\mbox{mas}(G_{(j)})}
\end{split}
\end{equation}
where $(a)$ follows from (\ref{eq:squareLoss}), $(b)$ follows from the fact that since the probability of selecting a new action is at most $\epsilon_{t_j}$, the mean and the variance of the geometric random variable $r_{j}$ is bounded by $1/\epsilon_{t_j}^2$ and $(1-\epsilon_{t_j})/\epsilon_{t_j}^2$ respectively, $(c)$ follows from the value of $\eta_{t}$ and $\epsilon_{t}$, and $(d)$ follows from the fact that $\mbox{mas}(G_{(j)})/\mbox{mas}(G_{(T)})$ is a monotonic non increasing sequence in $j$, therefore the summation is a concave function and the inequality follows from the Jensen's inequality. 

Now, we bound the $\mathbf{E}[\sigma(T)]$ in (\ref{eq:horizonFree}). This also gives a bound on the number of switches performed by the algorithm. We have
\begin{equation}
\begin{split}
\mathbf{E}[\sigma(T)] &=\sum_{t=1}^{T}\mathbf{E}[\mathbf{1}(i_{t}\neq i_{t-1} )],\\
&\leq \sum_{t=1}^{T}\epsilon_{t},\\
&\leq 0.5\mbox{mas}^{1/3}(G_{(T)})T^{2/3}c^{1/3}
\end{split}
\end{equation}
\end{proof}
\end{document}